\documentclass{article}

\usepackage{PRIMEarxiv}

\usepackage[utf8]{inputenc} 
\usepackage[T1]{fontenc}    
\usepackage{hyperref}       
\usepackage{url}            
\usepackage{booktabs}       
\usepackage{amsfonts}       
\usepackage{nicefrac}       
\usepackage{microtype}      
\usepackage{lipsum}
\usepackage{fancyhdr}       
\usepackage{graphicx}       
\graphicspath{{media/}}     

\usepackage{natbib}
 \bibpunct[, ]{(}{)}{,}{a}{}{,}%

\usepackage[parfill]{parskip}
\usepackage{multirow}
\usepackage[noend, ruled]{algorithm2e}
\usepackage{romannum}
\usepackage{placeins}
\usepackage{csquotes}
\usepackage{paralist}
\usepackage{xcolor}
\usepackage[onehalfspacing]{setspace}
\usepackage{appendix}

\RequirePackage{amsmath,amssymb,ifthen,array,theorem}
\def\argmin{\mathop{\rm arg\,min}}%

\newtheorem{theorem}{Theorem}[section]

\newtheorem{claim}{Claim}[section]

\theoremstyle{EX}

\newtheorem{definition}{Definition}[section]

\newenvironment{proof}{\paragraph{Proof:}}{\hfill$\square$}

\AtBeginDocument{\pagenumbering{arabic}}

\definecolor{safeblue}{RGB}{0,0,0}
\newcommand{\rev}[1]{{\color{safeblue} #1}}

\title{The Restaurant Meal Delivery Problem\\ with Ghost Kitchens}

\author{
  Gal Neria, Michal Tzur \\
  Department of Industrial Engineering \\
  Tel Aviv University \\
  Tel Aviv\\
  \texttt{\{galneria@mail, tzurm@tauex\}.tau.ac.il} \\
   \And
  Florentin D Hildebrandt, Marlin W Ulmer \\
  Chair of Management Science \\
  Otto-von-Guericke-Universität Magdeburg \\
  Magdeburg\\
  \texttt{\{florentin.hildebrandt, marlin.ulmer\}@ovgu.de} \\
}

\begin{document}
\maketitle

\begin{abstract}
Restaurant meal delivery has been rapidly growing in the last few years. The main challenges in operating it are the temporally and spatially dispersed stochastic demand that arrives from customers all over town as well as the customers' expectation of timely and fresh delivery. To overcome these challenges a new business concept emerged, \enquote{Ghost kitchens}. This concept proposes synchronized food preparation of several restaurants in a central complex, exploiting consolidation benefits. However, dynamically scheduling food preparation and delivery is challenging and we propose operational strategies for the effective operations of ghost kitchens. We model the problem as a sequential decision process. For the complex, combinatorial decision space of scheduling order preparations, consolidating orders to trips, and scheduling trip departures, we propose a large neighborhood search procedure based on partial decisions and driven by analytical properties. Within the large neighborhood search, decisions are evaluated via a value function approximation, enabling anticipatory and real-time decision making.
We show the effectiveness of our method and demonstrate the value of ghost kitchens compared to conventional meal delivery systems. We show that both integrated optimization of cook scheduling and vehicle dispatching, as well as anticipation of future demand and decisions, are essential for successful operations. We further derive several managerial insights, amongst others, that companies should carefully consider the trade-off between fast delivery and fresh food.
\end{abstract}


\section{Introduction}
The demand for restaurant meal delivery is booming. More and more people order food online and expect a fast and fresh delivery at low cost. The high expectations are often not met in practice, and customers complain about long waits and cold food. At the same time, restaurants and delivery platforms struggle to become profitable with meal delivery. One reason is the limited consolidation potential of orders. Customer orders come in over time from locations all over the city. Restaurants are distributed all over the cities as well. Combined with very tight delivery promises, this makes consolidation nearly impossible, or, when consolidated, food likely arrives not fresh at the customers. 
With these challenges in mind, new business concepts emerge, called \enquote{Ghost kitchens}, \enquote{Dark kitchens}, or \enquote{Cloud kitchens} \citep{shapiro2023platform}. All propose food preparation of several restaurants in one, central complex. The restaurants do not have dine-in customers, but exclusively prepare online orders that are distributed by a joint fleet of vehicles. 

Having several restaurants in one place and serving online orders only brings many advantages. For each restaurant, the kitchens can be designed and operated to allow fast and high-quality processing of the online orders \citep{feldman2022managing}. Further,  due to the same origin of the food, delivery consolidation can be achieved more often and likely without diminishing the food’s freshness. This is not the case for related concepts, also called ghost kitchens. In contrast to the concept considered in this work, these ghost kitchens are virtual restaurant brands that are prepared in partnering dine-in restaurants. They aim at maximizing the resource usage of the partner restaurant without cannibalizing its market share. This concept appeared during the COVID pandemic and is vanishing again due to its unpopularity rooted in long delivery times and unfresh meals \citep{ghostkitchens2024}. For the remainder of this work, ghost kitchens refer to delivery-only facilities where meals for multiple, virtual restaurants are prepared. For the ghost kitchen concept considered, an effective operation of food preparation and the vehicles' dispatching is necessary to allow fast and fresh delivery as well as consolidation. Delivery vehicles should be dispatched with food for customers in the same neighborhood. This avoids long travel and congestion of the delivery fleet. 

To make this possible, the restaurants need to prepare the corresponding food in a way that it is ready for delivery roughly about the same time. Otherwise, at least one customer will receive cold food. At the same time, the restaurant resources should be used efficiently to avoid congestion of orders and long waiting for the corresponding customers. This leads to complex decisions about integrated order sequences and vehicle dispatching with several constraints on freshness and synchronization. All this has to be done in real-time while more orders enter the system every minute.

In this research, we propose strategies for the effective operations of ghost kitchens. We model the problem as a sequential decision process  and demonstrate that the search over the complex, combinatorial decision space of integrated order scheduling and vehicle dispatching can be replaced by a search over decision representations of significantly reduced dimension. To further improve the search, we embed a novel polynomial algorithm that filters potential candidate decisions (of reduced dimension) via analytical feasibility checks. Then, a full-dimensional decision is generated if and only if feasibility is ensured.
Our search over the reduced decision space is performed by a 
novel large neighborhood search (LNS). Within the LNS, decisions are evaluated via value function approximation (VFA), enabling anticipatory and instant dynamic real-time decision. The VFA is defined independently of the instance size. It is trained on small instances and then transferred to larger instances via transfer learning. 

In our computational study, we compare the concept of Ghost kitchens to regular meal delivery systems with several independently operating restaurants. We further compare our strategy to several benchmark strategies for a wide range of instance settings. We derive the following insights: (1) Ghost kitchens bring significant advantages compared to conventional delivery systems with respect to service quality and workforce utilization. (2) Joint optimization of preparation and dispatching  improves customer experience by reducing the average delay of deliveries for all customers throughout the city.  (3) We observe an increase in bundling opportunities and a decrease in travel time. This does not only result from the elaborate search of the decision space. Anticipation of future orders and decisions is similarly important. (4) There is an explainable but potentially counterintuitive trade-off between the delivery speed and freshness of food. Fresh delivery may result in longer waiting for the customers while faster delivery may lead to less fresh food. (5) Sharing the fleet amongst restaurants is essential for successful operations of ghost kitchens. (6) A careful balance between the fleet and cook resources is crucial. A shortage in one resource cannot be compensated by the other.  (7) For restaurants, reducing both preparation time volatility and length are important factors for the customers' experience. 

The contributions of our work are both problem- and method-oriented. We are the first to analyze the operations of a new and innovative delivery concept that, compared to conventional meal delivery systems, has many advantages, which can be exploited via synchronization of the preparation and dispatching tasks. We demonstrate these advantages in our numerical study, where improved KPIs such as better food quality, higher speed of delivery and reduced travel times are presented. Based on these results we present valuable managerial insights. Methodologically, our contribution consists of first reducing the decision space to a significantly lower number of viable candidate solutions by: (i) Defining a novel reduced decision representation that can be searched more easily than the original one; (ii) Reducing the number of candidate decisions even further by filtering, based on analytical results, infeasible decisions; and (iii) Developing a novel LNS algorithm to search the remaining decision space, whose search steps are again based on analytical properties.  
Then, each candidate decision is evaluated via VFA, enabling the incorporation of its anticipated value instantly. In our experiments, we highlight the value of this combination and analyze the impact of balancing offline training effort with the online execution effort of the LNS-search. We note that the integration of scheduling and routing is rarely studied in a dynamic environment and that  our method can likely be adapted to other related problems with a scheduling and routing component, e.g., in the fields of order picking and dispatching or dynamic production routing, see the discussion in Section~\ref{sec: Discussion}.

The paper is organized as follows: In Section~\ref{sec: literature} we provide the literature review, in Section~\ref{sec:model} we define the problem, in Section~\ref{sec: our solution method} we present our solution method, in Section~\ref{sec:evaluation} we provide our experimental study, in Section~\ref{sec: Discussion} we discuss the generality of our methodology as well as its applicability to other methods, and, in Section~\ref{sec: conclusion} we summarize our work. The paper also provides additional details in an Appendix.

\section{Literature Review} \label{sec: literature}
Our work addresses the problem field of restaurant meal delivery, its model requires dynamic scheduling and routing, and the methodology focuses on quickly searching a large and complex decision space and evaluating decisions via value function approximation. In the following, we discuss related problem, model, and methodology literature.

\subsection{Problem: Meal Delivery Routing}

Work on restaurant meal delivery has surged in the last years. The majority of work aims on efficient and flexible routing strategies to deliver meals from different restaurants to dynamically requesting customers. Early work by \cite{reyes2018meal,steever2019dynamic,liu2019optimization} and \cite{ulmer2021restaurant} present intuitive heuristic methods to dynamically assign orders to vehicles. The general goal is to ensure timely service by balancing routing efficiency with fleet flexibility. \cite{yildiz2019provably} analyze optimal dispatching strategies in a deterministic setting deriving insights in the value of bundling operations and giving guidelines in demand management and the scheduling of delivery vehicles. Recently, first work on reinforcement learning was proposed by \cite{jahanshahi2022deep} to determine anticipatory assignments and rejections of orders, i.e., learning the expected future revenue when a decision is taken. However, scheduling or routing decisions were not considered. Other work on meal delivery addresses arrival time predictions \citep{hildebrandt2020supervised} with the goal of providing accurate delivery time information to customers, the scheduling of the workforce \citep{ulmer2020workforce,dai2020workforce,auad2022dynamic} to ensure timely delivery without excessive workforce cost, or the zoning of the service area with respect to current and future demand \citep{ulmer2022dynamic,auad2023courier}. While most of the studies focus on timely delivery for customers, none of the considered studies analyzes ghost kitchens or considers integrated scheduling and routing decisions. In the problem addressed in this paper, we consider the assignment and scheduling of orders in combination with the vehicle dispatching. To achieve efficiency (and sometimes even feasibility) in the overall process, decisions regarding order preparation and the vehicle dispatching operations need to be coordinated in an integrated and anticipatory fashion.
While synchronization in meal delivery has not been considered in the literature yet, the importance of synchronization in vehicle routing problems has been highlighted by \cite{drexl2012synchronization}, who present a survey on vehicle routing problems with synchronization constraints.

\subsection{Model: Combined Scheduling and Routing}

From a modeling perspective, our problem may be considered as a combination of models for dynamic order scheduling, e.g., \cite{xu2016stochastic}, \cite{zhao2018minimizing},  \cite{zhao2018stochastic} and \cite{d2023integrated}, and dynamic vehicle dispatching, e.g., \cite{klapp2018dynamic}. A combination of both is relatively rare, and we list here a few exceptions. 

\cite{moons2017integrating} review the integration of production scheduling and vehicle routing decisions. They mentioned very few papers on stochastic problems and no papers that considered dynamic decisions, which they pointed to be a future research direction. \cite{zhang2019online} consider the online integrated order picking and delivery problem for an online-to-offline (O2O) community supermarket. They proposed an online algorithm with a theoretical competitive ratio of 2. \cite{hossein2021rule} suggest a rule-based heuristic algorithm for order picking scheduling and delivery planning of online retailers with multiple order pickers. A stochastic dynamic arrival of orders is assumed, however, the delivery of products is performed from the warehouse to the (given) targeted cross-dock that services the corresponding customers. That is, no routing considerations are included. \cite{liu2022approximate} examine the integration of production and delivery when orders arrive dynamically. 
As in our problem, they assume that orders must be processed on one of a set of parallel machines before their delivery. Still, the delivery considered is to one mutual destination for all at given departure times. 
The authors proposed an approximate dynamic programming solution in which simple principles are used to reduce the decision space, e.g., shortest-processing time in production and first-in-first-out (FIFO) in delivery. Finally, \cite{rijal2023} suggest integrated optimization of warehouse operations and delivery routing. In a static and deterministic setting, they compare sequential and integrated approaches and highlight the potential of joint optimization. 

To conclude, the integration of scheduling and routing is rarely studied in a dynamic environment and usually without anticipation (i.e., consideration of future order arrivals when making decisions). The problem we address provides a more general setting that has become increasingly prevalent in the online delivery domain nowadays. Furthermore, we are among the first presenting an anticipatory policy that integrates scheduling and routing optimization.

\subsection{Methodology: Search and Evaluation of Complex Decision Spaces}

The methodological contribution of our work is to combine the search of a complex decision space with an evaluation via value function approximation (also known as reinforcement learning, RL). Value function approximations (and RL-methods in general) repeatedly simulate the decision process and store the observed values of decisions in an aggregated form. The values are used for decision making and updated over the simulation. As a recent review by \cite{HILDEBRANDT2023106071} shows, past research either focused on a comprehensive search without explicit evaluation or alternatively, an explicit evaluation via RL, but only for a very few potential decisions. 
There are, however, a few exceptions. 

\cite{rivera2017anticipatory,rivera2022anticipatory} and \cite{heinold2022primal} use RL to decide which freight to dispatch. The values are approximated via a linear function based on a set of post-decision state features. Mixed integer programming is used for searching the decision space. A similar concept is proposed by \cite{SILVA2023100105} to decide about assignments in urban delivery, however, a neural net is used for approximation. This neural net is again linearized to allow the application of mixed integer programming. None of the works considers more complex decision spaces and problems that involve routing decisions with complex constraints. This slows down the search of the decision space while at the same time makes linearization of the features and the value function approximation more challenging.

Recently, \cite{NeriaTzur2022} proposed using a metaheuristic instead to search the decision space quickly and evaluate decisions with a value function approximation. The general idea is similar to our approach. However, while our works share the general combination, our problem differs significantly from \cite{NeriaTzur2022} and the design of both search and evaluation methodology is therefore essentially different. 
For our problem, the decision space is very large and finding feasible decisions is a challenge. 
Thus, we propose an alternative decision formulation that reduces the burden to determine a complete decision and allows for fast feasibility checks based on a set of carefully 
derived propositions. With the more complex decision space and larger problem size, learning the values becomes more challenging as well. While the mentioned fast search reduces the training time significantly, simulating systems of realistic size and learning the values is still computationally expensive. Thus, we propose transfer learning, training our policy on small instances and transferring our trained policy to realistically sized instances.

\section{Problem Statement}\label{sec:model}

In this section, we introduce the Restaurant Meal Delivery Problem with Ghost Kitchens (RMD-GK). First, we describe the problem. Then, we illustrate the model's components with an example. Finally, we formally define the problem as a sequential decision process.

\subsection{Problem Description} \label{sec:problem description}
A ghost kitchen is a facility for the sole purpose of preparing delivery-only meals. It is partitioned into a set of kitchens, each with dedicated cooks, and uses a joint fleet of capacitated vehicles for delivery (e.g., cargo bikes or delivery cars). Each kitchen is associated with a fixed ghost restaurant that we refer to as a food type 
and prepares only orders from that food type.

Over the course of the day, customers order food from the restaurants. In our setting, each order includes only one of the food types, i.e., can be prepared by one of the ghost restaurants. Each order is associated with a known preparation time and a customer location. Once an order is prepared, a vehicle picks up the order and delivers it (immediately or after some waiting time) to the respective customer, with or without other customers' orders at the same trip. If orders are bundled, and more than one order is dispatched in the same trip, the trip also captures the sequence of orders.

The scheduling of the orders in the restaurants and the scheduling of the delivery vehicles is done by a central information system. This system maintains and updates a schedule over time. The schedule determines for each restaurant which cook prepares which order, and when. We refer to this as the cook schedule. 
It further decides about the bundling of orders to trips, their assignment to vehicles, and the departure times of each trip. We refer to this as the vehicle schedule. A schedule is only feasible if each order's freshness is guaranteed when arriving at the customers. Thus, a schedule must ensure for each order that the ready-to-door time, i.e., the difference between the ready time of the food at the restaurant and the arrival time at the customer, does not exceed an order-specific duration. 
E.g., in case a cook is ready but upon the order preparation completion, a vehicle might not be available to deliver it fresh, the 
starting time of the order preparation may need to be postponed
until a fresh delivery can be guaranteed. Customers expect fast delivery, ideally within 30 minutes after the order was placed. The goal of the platform is to meet the expectation by minimizing the average exceedance of delivery time (delay) over all orders.

\subsection{Example}

In the following, we provide a small example to illustrate the RMD-GK and to prepare the modeling. For the ease of presentation, we omit any notation and numbers.

\begin{figure}[t]
\centering
{\includegraphics[width=\textwidth]{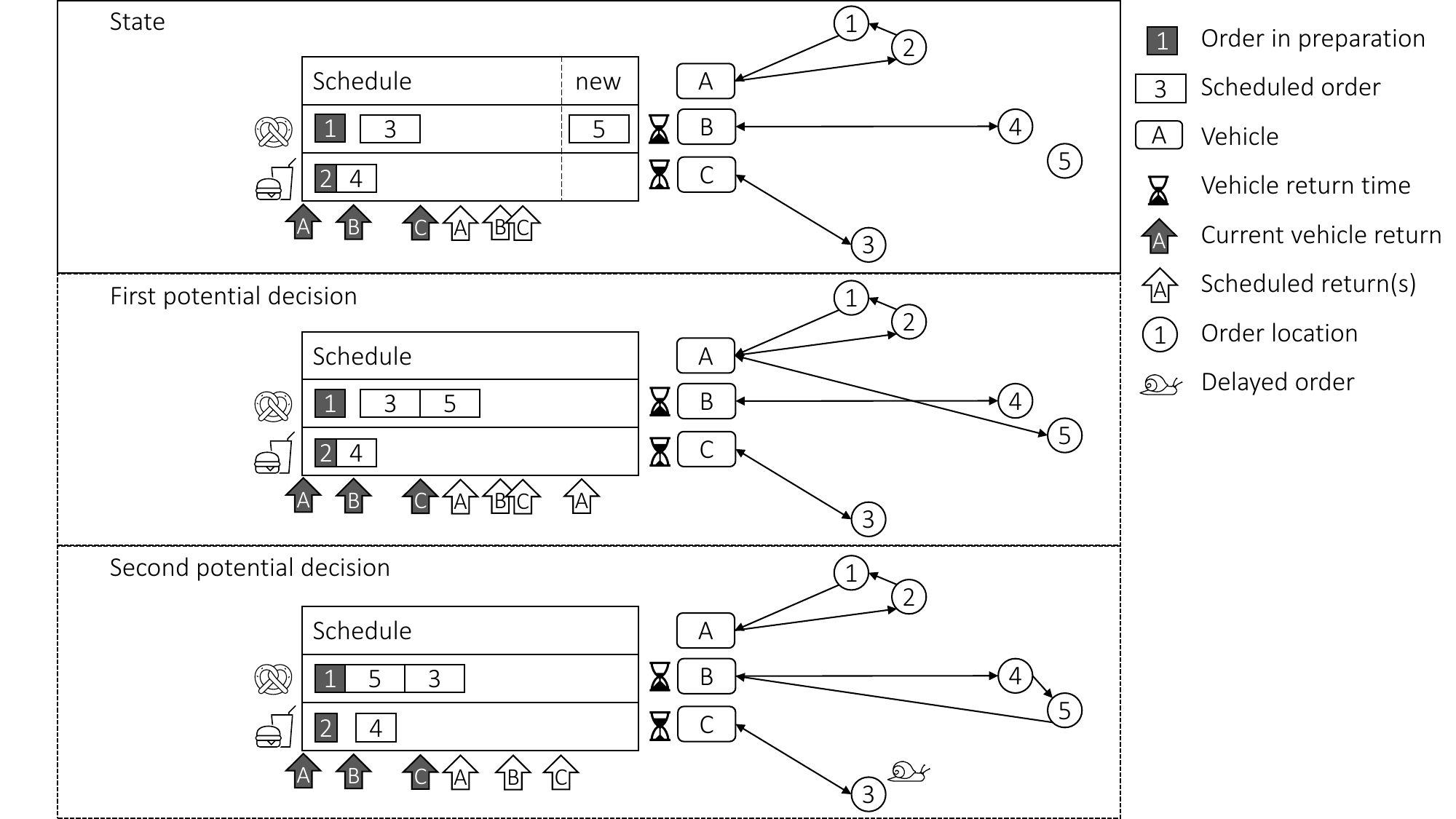}}
{Example for a state and two potential decisions.\label{fig:example}}
{}
\end{figure}

 The example is shown in Figure~\ref{fig:example}. 
 On the top of the figure, a possible state of the system is shown. 
 The middle and bottom parts depict the resulting system post-decision states for two potential decisions. 
 In the example, the schedules of the numbered orders in the kitchens are depicted in the large box on the left. There are two food types (German and US), each with one cook. 
In the depicted state, five orders are shown in rectangular shape. 
 The width of the shape indicates the preparation time for each order. Orders currently in preparation are depicted by the grey boxes. The size of each grey box corresponds to the remaining preparation time.
  Four of the orders (1, 2, 3, 4) are already scheduled, one is new (5). 
Orders 1, 3, and 5 belong to the first food type (German). Orders 2 and 4 belong to the second food type (US). Orders 1 and 2 are already in preparation. The current scheduling of the orders is (1,3) and (2,4) for the first and second food types, respectively. The gap between the finish time of order 1 and the start time of order 3 is the result of the vehicle departure plan, we discuss next.

 There are also three vehicles, denoted A, B, and C,  
 indicated by the three small squares 
 with soft corners. 
 The hourglasses beside the vehicles indicate that they are currently busy delivering orders. 
The geographical locations of the orders are shown on the right hand side. 
 In the depicted state, Vehicle A is currently available and Vehicles B and C are out for delivery. Since the details of their current  trips and deliveries do not matter for planning anymore, they are not depicted in the example. Only their return times matter, indicated by the hourglasses near the vehicles and depicted by the grey arrows in the schedules. Vehicle A just returned. Vehicle B returns earlier than Vehicle C. The planned bundles and trips for the vehicles are depicted on the right side of the figure. Specifically, it is planned that Vehicle A delivers orders 2 and 1 once they are finished. This sequence might be necessary to ensure freshness because preparation of order 2 finishes earlier than preparation of order 1. Once vehicle B (C) returns, it is scheduled to deliver order 4 (3). Because Vehicle C becomes available later, the start of preparation of order 3 is postponed. The scheduled later returns of the vehicles are depicted by the light arrows.

 In the middle and bottom parts of Figure~\ref{fig:example}, two potential decisions are shown. Each decision integrates the new order in a certain way and consequently updates the preparation schedules,  the bundles, and the trips of the vehicles.
 According to the first decision, the previous preparation schedule is kept and order 5 is added to it in a first-in-first-out (FIFO) manner. The same is true for the vehicles where order 5 is added as a direct trip to Vehicle A, leading to a second scheduled return for Vehicle A. An alternative decision is shown on the bottom part of the figure where order 5 is inserted before order 3. This allows joint delivery of orders 4 and 5, which are geographically close, by Vehicle B. To ensure feasibility, the preparation of order 4 is slightly postponed. Note that this decision causes a delay for order 3, indicated by the snail.

This small example already illustrates some of the complexities of decision-making and the trade-off of resources. Decisions need to consider scheduling and routing together to ensure freshness feasibility (and to minimize delays). Depending on the solutions, we can shift the workload between cooks and between cooks and vehicles. Further, we can sacrifice performance now for flexibility in the future. Switching orders 3 and 5 leads to less idling time for the cook of the first food type and earlier availability of Vehicle A. But it also leads to late delivery for order 3, later returns of Vehicles B and C, and it causes idling for the cook of the second food type.

\subsection{Sequential Decision Process}\label{sec: MDP plan based}
The RMD-GK is stochastic and dynamic.  It is stochastic, as orders are unknown until they are placed and their realizations follow known probability distributions in time and space. It is dynamic because decisions are made repeatedly over time. We model the problem as a sequential decision process. To that end, we first introduce the global problem notation. Then, we formally define the process components, namely, decision points, states, decisions, post-decision states, costs, stochastic information, transitions, and objective. 


We consider a \emph{capture} phase $[0, T^c]$ during which orders can be placed. We further define a longer \emph{operation} phase $[0, T]$ with $T>T^c$ sufficiently large to allow the fulfillment of all placed orders. We denote the set of all possible orders by $I$.  Recall that each order is associated with an order time, a certain food type, a known preparation time, and a customer location (the order's destination). The food types associated with orders are modeled via the set $F$.  Each food type has a corresponding freshness time $\delta_f, f\in F$, i.e., a maximum time allowed for the meal to
take from the time it is prepared until it arrives at the order’s destination. Furthermore, each order should be delivered within a promised delivery time $\tau\in\mathbb{R}$ from its placement time, and exceeding it leads to a delay cost. The set of cooks is denoted by $C$ with subsets $C_f$ working on each food type $f\in F$. The vehicles are modeled via a set $V$. The vehicle capacity, i.e., the maximum number of orders allowed to be simultaneously carried by each vehicle, is denoted by $\kappa\in\mathbb{N}$. The travel time from the location of order $i\in I\cup\{0\}$ (or the ghost kitchen) to the location of order $j\in I\cup\{0\}$ (or the ghost kitchen) is given by $t^t_{ij}\in\mathbb{R}_{\geq 0}$, where 0 denotes the location of the ghost kitchen.

We now present all elements of the sequential decision process, following the modeling framework of \citet{ulmer2020modeling}.

\begin{description}
    
    \item[Decision Points.] A decision point $k$ occurs whenever a new order is placed and at the end of the order capture phase at time $T^c$ when no more orders will arrive. As the realization of orders is stochastic, the number of decision points $K$ is a random variable.
       
    \item[States.] 
A state $S_k\in\mathcal{S}$ comprises all information relevant for decision-making.
We categorize the state information according to \emph{order information}, \emph{currently planned cook schedules}, and \emph{currently planned vehicle schedules}.
First, we describe the order information. Except for the final state $S_K$ in $T^c$, a state contains a new order. The new order in state $S_k$ at time $t_k$ is denoted by $i_k$. The set of \emph{open orders}, i.e., orders that have been placed but have not yet left the ghost kitchen for delivery, is denoted by $I_k$ and includes order $i_k$. Each order $i \in I_k$ is represented by four variables indicating 
the food type $f_{i} \in F$ of order $i$; the time of day at which order $i$ was placed $t^o_{i}\in [0,T^c]$; the preparation time of the corresponding meal $t^p_{i}\in\mathbb{R}_{\geq 0}$; and the location of the order $l_i$.
In summary, each order $i\in I_k$ is associated with $(f_i, t^o_i, t^p_i, l_i)$.
\\
Second, we describe the currently planned cook schedules. The cook schedules were decided on in the last decision point $k-1$ and pruned in the transition to the current state (see \emph{Transition}). They include all orders in $I_k$ \rev{except} the new order $i_k$. In state $S_k$, the currently planned sequence of orders prepared by cook $c\in C$ is given by 
$\psi_{kc}=(i_{1}^c, i_{2}^c,\dots)\subset I_k\setminus\{i_k\}$.
The set of all preparation sequences is given by $\Psi_k=\{\psi_{kc}\mid c\in C\}$.
The planned time of day at which the preparation of order $i\in I_k\setminus \{i_k\}$ is started is denoted by $t^s_{ki}\in[t^o_i,T]$ where $s$ is a symbol. The corresponding set is given by $t^s_k=\{t_{ki}^s\mid i\in I_k\setminus\{i_k\}\}$. 
In summary, the state variables corresponding to the cook schedules are given by $(\Psi_k, t^s_k).$\\
Third, we describe the currently planned vehicle schedules. Analogous to the cook schedules, the currently planned vehicle schedules were decided on in the last decision point and pruned in the transition to the current state (see \emph{Transition}). The schedules contain all orders in $I_k$ \rev{except} 
the new order $i_k$. The sequence of trips performed by vehicle $v\in V$ is denoted by $\Theta_{kv}=(\theta_{kv1},\theta_{kv2},...)$. Each trip $\theta\in\Theta_{kv}$ consists of a sequence of orders $(i_{1}^\theta, i_{2}^\theta,\dots)\subset I_k\setminus\{i_k\}$. The set of all planned trip sequences is given by $\Theta_k=\{\Theta_{kv}\mid v\in V\}$.
 The planned time of day at which trip $\theta\in\Theta_k$ departs is denoted by $t^d_{k\theta}\in[t_k,T]$ where $d$ is a symbol. The set of all departure times is denoted by $t^d_k=\{t^d_{k\theta}\mid \theta\in\Theta_k\}$. Each vehicle's $v\in V$ (next) return time to the ghost kitchen is denoted by $t^r_{kv}\in[t_k, T]$ ($r$ is a symbol), where its value is either $t_k$ if the vehicle is idling or its return time from its last trip that departed before $t_k$ (note that the trip itself is not part of the state $S_k$). The set of all vehicles' return times is given by $t^r_k$. 
To summarize, we represent a state as the tuple

        \begin{equation}
            S_k=\big(\underbrace{t_k, I_k}_{\text{Orders}}, \underbrace{ \Psi_k,t^s_k}_{\text{Cooks}}, \underbrace{\Theta_k, t^d_k, t^r_k}_{\text{Vehicles}}\big).
        \end{equation}

    \item[Decisions.] 
    A decision $x_k\in\mathcal{X}(S_k)$ is an update $x_k=\big(\Psi_k^x, t_k^{s,x}, \Theta_k^x, t^{d,x}_k)$ on the current plan $\big(\Psi_k, t^s_k, \Theta_k, t^d_k\big)$.
    This update must integrate the new order $i_k$ into the current plan (except for the final state $S_K$ in $T^c$) but may also change other parts of the plan. 
    For example, it might reassign and reschedule orders to cooks or vehicles to adapt to changes in information. 
    In general, an update on the current plan is feasible only if\rev{:} 
    \rev{\begin{compactitem}
        \item Each order is scheduled for preparation and delivery by exactly one cook and one vehicle; 
        \item The freshness constraint is satisfied, i.e., the time between the end of preparation and scheduled time of delivery of an order \rev{$i$} is sufficiently small \rev{(i.e., less than $\delta_{f_i}$)};
        \item Already started preparations remain unchanged; 
        \item  Each order's preparation finishes before its delivery trip departs;
        \item Only one order is prepared by a cook at a time;
        \item Each trip of a vehicle departs only after the return from the previous trip; 
        \item Vehicles satisfy the capacity constraint;
        \item 
        Trips are elementary tours that start and end at the ghost kitchen.
    \end{compactitem}
    }    
    The freedom of modifying the previous schedule combined with the various problem constraints yield a complex decision space $\mathcal{X}(S_k)$. For this reason, we model the decision space $\mathcal{X}(S_k)$ as a mixed integer linear program (MILP) in Appendix~\ref{app: decision space}.

\item[Post-Decision States.]
The post-decision state $S_k^x\in\mathcal{S}^x$ represents the state directly after a decision is made but before it transitions to the next state (before the arrival of the next order $i_{k+1}$). Thus, given the state $S_k=\big(t_k, I_k, \Psi_k, t^s_k, \Theta_k, t^d_k, t^r_k\big)$ and the decision $x_k=\big(\Psi_k^x, t_k^{s,x}, \Theta_k^x, t^{d,x}_k\big)$, the post-decision state is defined as $S_k^x=\big(t_k, I_k, \Psi_k^x, t^{s,x}_k, \Theta_k^x, t^{d,x}_k, t^{r}_k\big)$.

\item[Costs.] 
The {immediate cost}
$D^\Delta_k$ associated with state $S_k$ and decision $x_k=(\Psi_k^x, t^{s,x}_k, \Theta_k^x, t^{d,x}_k)$
is
given by the marginal difference in cost of the old plan $(\Psi_k, t^s_k, \Theta_k, t^d_k)$ and the updated plan $(\Psi^x_k, t^{s,x}_k, \Theta^x_k, t^{d,x}_k)$ (\rev{see} \citealt{ulmer2020modeling}). 
The cost of a plan corresponds to the total \enquote{planned} delay of the orders, i.e., the delay that would be observed when following the plan. Let 
$D(\Theta, t^d)$ denote the total delay for a given plan $(\Psi, t^s, \Theta, t^d)$. We define $D$ according to
\begin{equation}
    D(\Theta, t^d)=\sum_{\theta\in\bigcup\Theta}\ \sum_{j=1}^{|\theta|} \max\Bigg(0, t^d_\theta + \left(\sum_{l=1}^{j-1} t^t_{i^\theta_{l}, i^\theta_{l+1}} \right) - t^o_{i^\theta_{j}} - \tau\Bigg),
\end{equation}
where $\bigcup\Theta$ denotes the union of all sets of trips contained in $\Theta$.
In words, we sum over all trips of all vehicles. For each trip, we compute the delay of each order as the maximum of 0 and the difference of the time span from order placement to order arrival and the promised delivery time $\tau$. We define the marginal cost {(the immediate cost)} $D^\Delta_k$ as
\begin{equation}\label{eq: immediate cost}
    D_k^\Delta(S_k, x_k)=D(\Theta_k^x, t^{d,x}_k)-D(\Theta_k, t^d_k).
\end{equation}
\rev{Note, that the marginal costs can be negative in case the delay of the new plan is smaller than the one of the previous plan.} Further note, that the sum of marginal costs over all realized states coincides with the sum of realized delays over all orders. In the special case of the final state $S_K$ in $T^c$, the process terminates after the decision was made.

\item[Stochastic Information.] 
The stochastic information $W_{k+1}\in\Omega$ is given by the new order $i_{k+1}$ and corresponding food type $f_{i_{k+1}}$, order placement time $t^o_{i_{k+1}}$,  meal preparation time $t^p_{i_{k+1}}$ and order location $l_{i_{k+1}}$. 
 There is a special case where $W_{k+1}=\emptyset$, i.e., no more orders are realized. In that case, the sequential decision process moves to the final state $T^c$.
   
\item[Transition.] 
The transition $S^M: \mathcal{S}\times\mathcal{X}\times\Omega\rightarrow \mathcal{S}$ maps a state-decision tuple and stochastic information to the next state $S_{k+1}=S^M(S_k, x_k, W_{k+1})$. The current time is updated to $t_{k+1}=t^o_{i_{k+1}}$. 
The new set of open orders is given by 
$I_{k+1}=\displaystyle\bigcup_{\theta\in\Theta^x_k}\{i\in \theta\mid t^{d,x}_{k\theta}>t_{k+1}\}\cup\{i_{k+1}\}$, i.e., all orders whose departure times are after $t_{k+1}$. 
The schedule of each cook $c\in C$ is truncated to include only orders in $I_{k+1}$.
The new planned sequence of trips for each vehicle $v\in V$ is given by truncating trips $\theta\in\Theta_{kv}^x$ with $t^{d,x}_{k\theta}<t_{k+1}$. 
The return time of each vehicle $v\in V$ is updated according to

\begin{equation}    t^{r}_{k+1,v}=\max\left(t_{k+1},\max_{\theta\in \Theta^x_{kv} \mid t^{d,x}_{k\theta}<t_{k+1}} t^{d,x}_{k\theta}+ t^t_{i^\theta_{|\theta|},0} + \sum_{j=1}^{|\theta|-1}t^t_{i^\theta_{j}, i^\theta_{j+1}}\right)
\end{equation}


\item[Objective.] 
A mapping that assigns each state $S_k$ a decision $\pi(S_k)=x_k$ is called a policy and denoted by $\pi\in\Pi$. Given an initial state $S_0=(0, \emptyset, \emptyset, \emptyset, \emptyset, \emptyset, \emptyset)$, the objective is to find an optimal policy $\pi^*$ that minimizes the expected sum of delay. This coincides with minimizing the expected marginal costs over all decision points when starting in $S_0$ and applying policy $\pi$ throughout the process:

\begin{equation}
   \min_{\pi\in\Pi}\mathbb{E}\Big[\sum_{k=0}^K D^\Delta\big(S_k, \pi(S_k)\big)|S_0\Big].
\end{equation}
\end{description}

\section{Solution Method} \label{sec: our solution method}


In this section, we present our solution method. 
We first give a motivation and an overview in \rev{Section} \ref{sec:motivation and overview} before defining the two main components, search (\rev{Section} \ref{sec: search the action space}) and evaluation (\rev{Section} \ref{sec: LNS Solutions Evaluation}), in detail.

\subsection{Motivation and Overview}\label{sec:motivation and overview}

The problem at hand requires fast and effective decisions. The small example in Figure~\ref{fig:example} already illustrates the complexity of finding a feasible and effective decision for cook and vehicle schedules. \rev{Even for this small example, the number of potential decisions is already vast. Also, it is not clear which of the two decisions shown in the example is more effective as it depends on future orders and decisions. The first one avoids delay now, but binds more vehicle resources while the second causes a slight delay now, but saves vehicle resources. The first might be advantageous if the expected number of future orders is small while the latter becomes more effective in case many orders can be expected in the future. Hence, the evaluation of a decision should \emph{anticipate} potential future delay when a specific decision is taken. The two challenges of a thorough search of the decision space and an anticipatory evaluation of decisions are captured in the Bellman Equation:}\vspace{-1cm}

\begin{equation}\label{eq:optimal}
    \pi^*(S_k)\in \rev{\argmin}_{x\in\mathcal{X}(S_k)}D^\Delta(S_k, x) +\mathcal{V}(S_k^x).
\end{equation}

The Bellman Equation defines an optimal decision $\pi^*(S_k)$ from the set of overall decisions $\mathcal{X}(S_k)$ in a state $S_k$ based on the (known) immediate cost $D^\Delta(S_k, x)$ and the expected cost-to-go when taking a decision $x$. The cost-to-go is 
modeled via the value function $\mathcal{V}$ mapping post-decision states $S_k^x$ to the expected future cost $\mathcal{V}(S_k^x)$:\vspace{-0.5cm}

\begin{equation}
    \mathcal{V}(S_k^x) = \mathbb{E}\Big[\sum_{k^\prime=k+1}^K D^\Delta\big(S_{k^\prime}, \pi^*(S_{k^\prime})\big)|S_k^x\Big].
\end{equation}

Thus, finding an effective policy poses two main challenges. First, a comprehensive, but runtime-efficient search of the complex combinatorial decision space is required, indicated by $\rev{\argmin}_{x\in\mathcal{X}(S_k)}$ in the Bellman Equation. Second, the decisions need to be evaluated instantly with respect to their immediate and future impact, with the latter represented by the (unknown) value function $\mathcal{V}(S_k^x)$ in the Bellman Equation. We address these challenges by combining a large neighborhood search (LNS) of the decision space driven by analytical insights with a value function approximation (VFA) tuned via transfer learning.  

Specifically, we propose an LNS to search the decision space. Since the decision space is vast, a large number of LNS-iterations are required to search it thoroughly. 
However, many decision candidates turn out infeasible in the end due to the capacity or freshness constraints and need to be discarded. 
To allow a fast but thorough search and quickly identify infeasible decision candidates, we base our LNS on a \emph{condensed} representation of a decision, i.e., one that reduces the decision space without loss of optimality. 
Within the LNS we search on condensed \emph{partial} decisions, i.e., where the planned sequence of orders at the cooks and trips of vehicles is determined but their exact timing is not. 
To determine feasibility and timing of the corresponding (partial) sequencing decision, we propose a polynomial algorithm, based on a dynamic programming formulation and a set of analytical propositions.

To evaluate a feasible decision candidate,  we propose a VFA. The VFA approximates values $\mathcal{V}(S^x_k)$ via repeated simulations. Once the values are learned, decisions can be evaluated instantly. For the VFA, a post-decision state is represented by a set of features. As illustrated in the example in Figure~\ref{fig:example}, the value depends on the resources' availability and their balance \rev{(i.e., the workload distribution among cooks and vehicles, both individually and relative to one another)}. Thus, we derive problem-specific features accordingly. A final challenge of the VFA is the runtime. Even though training can be done offline, the training requires a large number of simulations with many states and decisions via LNS. Thus, we propose transfer learning, training on a smaller problem size and then adapting the trained policy to larger-size instances via fine-tuning. As approximation architecture, we rely on a neural network (NN). The general idea of combining a neighborhood search with a NN-based VFA was proposed in \citet{NeriaTzur2022}.

\rev{As our method allows for an integrated optimization of scheduling and routing while anticipating future orders, we denote it \textit{Anticipatory Integrated} (AI).}
The overall procedure when a new order arrives to the system is summarized in Figure~\ref{fig:LNS}. In a state, an initial (original) decision is fed to the LNS as the current best decision. Then, over a number of iterations, a new, partial (condensed) decision is generated by a set of operators. 
The feasibility of the partial decision is checked, and if not feasible, it is discarded and a new partial decision is generated. 
This is done by our proposed \emph{PDFT} (Partial Decision Feasibility and Timing)-algorithm. In case a partial decision is feasible, the PDFT creates the full (original) decision, and the value of the decision, provided by the stored approximated values, is compared to the current best decision. In case the value is better, the currently best found decision is updated. After a given number of iterations, the procedure returns the best found decision. This decision is then implemented in the sequential decision process. In the following, we describe the individual steps of the procedure in detail.

\begin{figure}[t]
    \centering
    \caption{Overview of the steps of our \rev{\textit{AI}-method} in a decision state}
    \includegraphics[width = \textwidth]{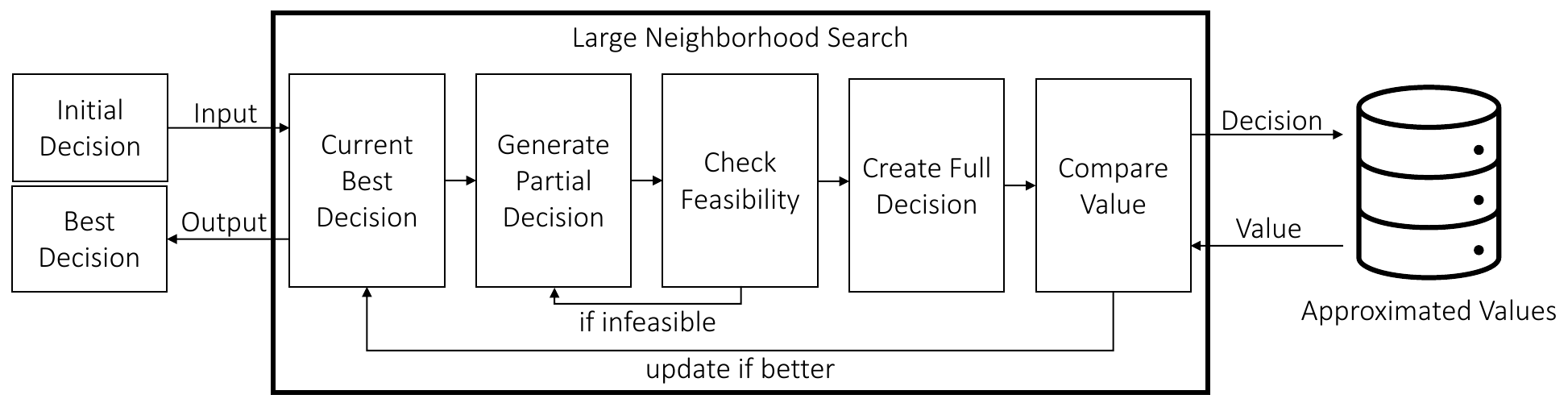}
    \label{fig:LNS}
    \vspace{-1.0cm}
\end{figure}

\subsection{The Large Neighborhood Search} \label{sec: search the action space}

In this section we describe our 
method to search the decision space whenever a new order arrives in the system and a decision needs to be made. We follow the steps proposed in Figure~\ref{fig:LNS}. 
We first introduce the condensed decision representation and, building on it, the notion of a partial decision. We describe the LNS-operators for the partial decision space. We then present the PDFT-algorithm to check feasibility and create the full decision. 

\subsubsection{A Condensed Representation of a Decision and a Partial Decision.}\label{sec: condensed rep}

We start with a formal definition of a condensed decision representation. To that end, we denote by $I_{kf}=\{i \in I_k \mid f_i=f\}$ the set of open orders of food type $f\in F$. 
Recall that $x_k$ includes assignments of orders to specific cooks and assignments of trips to specific vehicles. The idea of a condensed decision is to combine together all orders that have the same food type (all trips) instead of assigning each order to a specific cook of the respective food type (specific vehicle). 

\begin{definition} [A condensed decision representation]  \label{def: condensed solution representation}
A \emph{condensed decision representation} of $x_k$ includes the following components. 
For each food type $f \in F$, a sequence of processing all orders in $I_{kf}$ is given by $\psi_{k}^{f,x}=(i_{f1}, i_{f2},\dots)$. 
The set of all preparation sequences (for all food types) is given by $\rev{\widehat{\Psi}_k^x}=\{\psi_{k}^{f,x}\mid f\in F\}$. 
 The sequence of trips to dispatch by all vehicles in $V$ is given by a single vector $\rev{\widehat{\Theta}_k^x} = (\theta_1, \theta_2, \dots)$.
 The associated times $t^{s,x}_k$ and $t^{d,x}_k$ to start each order and depart each trip, respectively, remain as in the original decision representation (Section~\ref{sec: MDP plan based}).
\end{definition}

That is, $\widehat{\Psi}_k^x$ replaces $\Psi_k^{x}$ in the original decision representation by storing the sequences of processing all orders by food types instead of by specific cooks. 
Similarly, $\widehat{\Theta}_k^x$ replaces $\Theta_k^{x}$ in the original decision representation by storing a sequence for all vehicles in $V$ instead of a sequence for each vehicle. Note that because the times to start each order and depart each trip are identical for the original and its associated condensed representations, they have the same cost. We define a condensed decision to be feasible if there exists an original feasible decision that induces the condensed representation. 
 With this representation we reduce the decision space by eliminating equivalent decisions in $\Psi_k^x$ and $\Theta_k^x$ (because of the homogeneity of cooks of the same food type and vehicles). For the first decision in the small example in Figure~\ref{fig:example}, the representation of the food type sequences is identical to the cook schedule representations, i.e., (1,3,5) and (2,4) for the German and US food types, respectively,
since only one cook per food type is available. The representation also contains the corresponding start times for the orders (not explicitly stated in the example). In contrast to the original representation, the vehicle schedule would be represented as $((1,2),(4),(3),(5))$, again, with the corresponding departure times. We next present definitions and claims that are used to prove Theorem \ref{th: condensed representation} below, that the condensed decision space also contains an optimal decision.
\begin{definition} [Symmetric decisions]\label{def: symmetric decisions}
Given a state $S_k$, two decisions $x$ and ${x}^\prime$ are symmetric if they include the same set of trips as well as the same sets of starting times $t^{s,x}_k$ and $t^{d,x}_k$ for orders and trips. 
\end{definition}
Two symmetric decisions can only differ in the specific cook and vehicle assignments.


\begin{claim} [Condensed representation of symmetric decisions]\label{th: symmetric post decision states have the same partial decisions} 
Given a state $S_k$, 
decisions are represented by the same condensed decision representation if and only if they are symmetric decisions.
\end{claim}
\begin{proof}{Proof of Claim \ref{th: symmetric post decision states have the same partial decisions}.}
Immediately derived from Definitions \ref{def: condensed solution representation} and \ref{def: symmetric decisions}.
\end{proof}

\begin{claim} [Symmetric decision costs]\label{th: costs symmetric post decision states}
Given a state $S_k$, the resulting post-decision states of symmetric decisions have the same immediate cost and the same cost-to-go.
\end{claim}

\begin{proof}{Proof of Claim \ref{th: costs symmetric post decision states}.}
In each pair of symmetric decisions 
the starting preparation and departure times of all orders and trips, respectively, are identical and hence the 
arrival time of each order $i \in I_k$ to the customer is identical, i.e., the total delay is the same (the immediate cost).
For the same reason, at each point of time $t \geq t_k$,
the number of available cooks of each food type and the number of available vehicles is the same in both post-decision states. As cooks and vehicles are homogeneous and their indexes are therefore interchangeable, the cost-to-go is also identical.
\end{proof}



\begin{theorem}\label{th: condensed representation}

Given a state $S_k$, an optimal decision can be found by searching over all (feasible) condensed decision representations rather than over all original decision representations.

\end{theorem} 

\begin{proof}{Proof of Theorem~\ref{th: condensed representation}.}
Every original decision representation has a corresponding condensed decision representation, and
every feasible condensed decision is associated with a non empty set of (feasible) original decision representations. 
Thus, all original decision representations are included in the search, and only feasible original decision representations are considered.
Moreover, all original representations that are associated with the same condensed representation are symmetric and have an equal immediate cost and cost to go (Claims \ref{th: symmetric post decision states have the same partial decisions} and \ref{th: costs symmetric post decision states}).
\end{proof}

Consequently, in our search heuristic we use the condensed decision representation to reduce the decision space. Towards a further reduction of the decision space, 
we next present the notion of a \emph{partial decision}, which reduces the decision space we have to explore at any given state.

\begin{definition} [A partial decision] \label{def:partial solution}
A partial decision includes only the sequences parts of a condensed decision, i.e., the sets of preparation and trip sequences, $\widehat{\Psi}_k^x$ and $\widehat{\Theta}_k^x$, respectively.
\end{definition}
For the example in Figure~\ref{fig:example}, this would leave the food type sequences (1,3,5) and (2,4) for the German and US food types, respectively, without the preparation starting times. It would also contain the condensed vehicle sequence $((1,2),(4),(3),(5))$, but without the departure times.

\subsubsection{The LNS-Procedure.}\label{sec: LNS} We now define our LNS-procedure operating on partial decisions. We start by generating an initial decision \rev{using a first-in-first-out (FIFO) procedure (Algorithm \ref{alg:FIFO} in Appendix \ref{app:Algorithm Pseudo Codes})}. Then, we generate new partial decisions using operators that alter task sequences in $\widehat{\Psi}_k^x$ and $\widehat{\Theta}_k^x$, as described below.
The feasibility of each partial decision is checked using a polynomial algorithm, the PDFT \rev{(Appendix \ref{app: PDFT})}. 
If a partial decision is feasible, this algorithm also assigns tasks to cooks and vehicles as well as task starting times.
\rev{A pseudo code of the LNS is given by Algorithm \ref{alg:LNS} in Appendix \ref{app:Algorithm Pseudo Codes}.}

The FIFO procedure generates the \emph{initial decision} by appending the new order $i_k$ to the planned schedule $\big(\Psi_k, t^s_k, \Theta_k, t^d_k\big)$ (as defined by the state variables of $S_k$), i.e., by assigning $i_k$ to a cook and a vehicle trip.
It starts by assigning order $i_k$ to the first available cook $c \in C_{f_{i_k}}$, initially at the earliest time $t \geq t_k$ in which $c$ is first available after the last scheduled order in $\psi_{kc}$ is prepared (according to the plan in $\Psi_k$ and $t^s_k$). 
Then, it searches for all vehicles $v \in V$ that have less than $\kappa$ orders assigned on their last scheduled trips according to $\Theta_k$ and their departure time is no earlier than $t+t_i^p$, the time order $i_k$ can be ready.
The selected trip is the one that can get $i_k$ to the customer with minimum tardiness. 
We note that if order $i_k$ joins a tour with other orders, the trip is determined to minimize the total tardiness over the orders included in it (and $i_k$) such that it is feasible.
If there is no existing scheduled trip the new order can join, then, it is assigned on a new trip that will depart as soon as both the order finishes preparation and there is an available vehicle.
If the earliest time, a vehicle can reach $i_k$ extends its freshness constraint, then, the time of preparation by cook $c$, $t$, is postponed such that the order will arrive fresh. The described procedure to insert order $i_k$ to the planned schedules $\big(\Psi_k, t^s_k, \Theta_k, t^d_k\big)$, without additional search steps, is used in Section~\ref{sec: test heuristic} as a first benchmark, that we refer to as \emph{FIFO}. This heuristic represents what is often done in practice.
\rev{A pseudo code of the \emph{FIFO} benchmark is given by Algorithm \ref{alg:FIFO} in Appendix \ref{app:Algorithm Pseudo Codes}.}

Based on the initial decision (or \emph{current decision}), new decisions are generated to search the decision space, as follows.
Let $\psi_{k}^{f,x}=(i_{f1}, i_{f2},\dots), \forall f \in F$ and $\widehat{\Theta}_k^x$ be the partial decision of the current decision, referred to as the \emph{current partial decision}. 
We generate the next decision by using one of seven operators to alter one of the sequences in $\psi_{k}^{f,x}, \forall f \in F$ or $\widehat{\Theta}_k^x$ such that we obtain a new partial decision, represented by new sequences: ${\psi_{k}^{f,x^\prime}}, \forall f \in F$  and $\widehat{\Theta}_k^{x^\prime}$ (where $\widehat{\Psi}_{k}^{x^\prime}=\{{\psi_{k}^{f,x^\prime}}, \forall f \in F\}$). 
\rev{The operators are randomly selected with equal probabilities.}
The operators were chosen according to our understanding of the problem domain, based on analytical properties as well as random procedures that assure that a large portion of the decision space can be covered. The individual operators are described in detail in Appendix \ref{App: operators}. Generally, \rev{Operators} 1 and 2 make changes in the order preparation sequences; Operators 3 and 4 change the trips sequence; and Operators 5-7 change the vehicle trips. 

After an operator is applied, we check if the new partial decision is feasible using the PDFT, which we describe in the next section. If not, we discard it. Else (if it is feasible), the PDFT also derives timing sets ${t^{s,x}_k}^\prime$ and ${t^{d,x}_k}^\prime$. 
If the decision has a lower (cost) value than the current one (see Section~\ref{sec: LNS Solutions Evaluation}), it becomes the new current partial decision from which the search continues.

\subsubsection{The PDFT.}\label{sec: PDFT}
We next summarize our PDFT-algorithm (for all details, we refer to Appendix~\ref{app: PDFT}). The input of the PDFT is a state $S_k$ and a partial decision, ${\psi_{k}^{f,x}}, \forall f \in F$ and $\widehat{\Theta}_k^x$. Its output is whether the partial decision is feasible or not, and, if yes, the starting times of order preparations and trip departure times.
To describe the PDFT we first sketch the problem it is aimed to solved, referred as the Assignment and Timing Sub-Problem (ATP): Given a state $S_k$ in the RMD-GK problem and a partial decision, i.e., set of sequences of orders to prepare for each food type and a sequence of trips to dispatch, ${\psi_{k}^{f,x}}, \forall f \in F$ and $\widehat{\Theta}_k^x$, respectively; the goal of the ATP is to assign the orders (trips) to the cooks (vehicles) and determine feasible task starting times such that the orders and trips are processed according to the given sequences and the total delay of the open orders in $I_k$ (given by $S_k$) is minimized. In Appendix~\ref{app: PDFT} we present a dynamic programming (DP) formulation of the ATP. We further present a set of analytical propositions that narrow down the state and decision space of the ATP without loss of optimality. We also present properties of the optimal solution to the DP formulation, which form the basis to the PDFT-algorithm.

The PDFT starts with the first state of the DP of the ATP and assigns orders and trips to cooks and vehicles, respectively, as soon as possible, one after another according to the obtained states along the DP. 
When obtaining infeasible states in the ATP because of freshness and/or synchronization between orders in the same trip and/or the given sequences' constraints (according to the partial decision), it returns back to the first assigned order that its starting time could have caused the state's infeasibility.
Then, it postpones only as much as necessary this order and some of the orders that followed it (along the DP) as well as some of the trips, so that the infeasible state becomes feasible (if possible).
This means that the final result of the PDFT is that all trips depart as soon as possible given state $S_k$ and the partial decision, which minimizes the delay. 
We design stopping conditions when no feasible solution exists and observe that the PDFT determines infeasibility of a partial decision if and only if it is infeasible, and otherwise, it solves the ATP to optimality. In our computational experiments, we limit the number of iterations to 25, however, in about $99\%$ of cases the PDFT terminates earlier, either because infeasibility was confirmed or because a feasible decision was found. \rev{A detailed analysis of the functionality of PDFT can be found in Appendix~\ref{sec: analysis heuristic}}.

\subsection{Decision Evaluation}
\label{sec: LNS Solutions Evaluation}

Ideally, we would like to evaluate each feasible decision $x_k$ in the LNS exactly, i.e., with respect to the sum of its immediate cost $D^\Delta_k(S_k, x)$ and the expected cost-to-go of the post-decision state, 
$\mathcal{V}(S_k^x)$. 
While the immediate cost can be computed directly as defined in Equation~\ref{eq: immediate cost}, the value function is unknown and may only be approximated. 
As decisions are evaluated not only in every state but also in every iteration of the LNS, it is crucial that the approximation of the value function is computationally fast. 
For that reason, we use a VFA given by a neural network $\widehat{\mathcal{V}}_\omega$, using a set of features associated with the post-decision state.
These features are based on aggregated values of the post-decision state.
In our case, the aggregated post-decision states are defined by summary statistics reflecting our available resources, i.e., the current status of cooks and vehicles (see Appendix~\ref{app: features}). The aggregation of the post-decision state is required for two reasons. First, the dimension of the post-decision state can be vast and, therefore, impede the approximation accuracy of the VFA, a phenomenon known as \emph{curse of dimensionality}. Second, we later aim to employ the same neural network for instances of different numbers of cooks and vehicles, a practice that is called \emph{transfer learning}, which requires the dimension of the network's input to remain fixed. Thus, aggregation of post-decision states of different size\rev{s} to the same features is required (\rev{In Appendix~\ref{sec: analysis heuristic}, we show the value of transfer learning}). 

The mapping of aggregated post-decision states to their estimated cost-to-go is defined by the neural network's weights $\omega$, fitted once in an extensive offline simulation. 
Afterwards, the fitted VFA $\widehat{\mathcal{V}}_\omega$ can be invoked directly,  when employed in the online simulation to evaluate the cost-to-go of a decision. To fit the network's weights $\omega$, we minimize the expected deviation of the neural network's estimate of the cost-to-go $\widehat{\mathcal{V}}_\omega(\mathcal{F}\big(S_k^x)\big)$ from the observed cost-to-go $\sum_{t=k+1}^K D^\Delta_k$ for each realized post-decision state $S_k^x$. The observed costs-to-go are only known in hindsight, i.e., after simulating the entire day. Thus, we only update the neural network's weights at the end of each simulated day. To do so, we save all observed tuples of aggregated post-decision states and its associated observed costs-to-go in a memory. 
When updating the neural network at the end of a day, we consider all previously observed tuples in the memory (observed during this day and previous days) by \rev{randomly selecting a batch of tuples from the memory}. 
This is done to avoid statistically dependent samples in the update step. Then, we compute the mean deviation (as given by the mean squared error) of the network's approximated value and the observed value. The gradient of the deviation with respect to our weights $\omega$ serves as a search direction to update the weights $\omega$ and, therefore, to improve the quality of our VFA. We repeat this offline simulation until the neural network's approximation quality converges. We then use the approximated values in the Bellman Equation when searching the decision space via LNS. In the special case of the last, deterministic decision point at time $T^c$, the future cost-to-go is zero and therefore, we set the value function to zero as well. We describe the details of the neural network \rev{features and} architecture and its update procedure in Appendix~\ref{app: architecture}.

\rev{The described LNS procedure to update the planned schedules $\big(\Psi_k, t^s_k, \Theta_k, t^d_k\big)$, without evaluation via VFA, is used in Section~\ref{sec: test heuristic} as a second benchmark, that we refer to as \emph{Integrated}. That is, the evaluation of each decision by this benchmark is made with respect to the immediate cost $D^\Delta_k(S_k, x)$ only.}

\section{Numerical Study} \label{sec:evaluation}
\rev{In this section, we describe instances and benchmark policies, analyze the performance of our method as well as the value of ghost kitchens, and derive managerial insights.}

\subsection{Instances}

In the following, we give an overview of our instances. 
For all experiments, we assume service in Iowa City with customer locations from \cite{ulmer2021restaurant}. The ghost kitchen hosts $|J|=5$ restaurants and is positioned in the center of Iowa City. Vehicles travel in the OpenStreetMap-network with free floating travel times \citep{boeing2017osmnx}. We assume orders can be placed the entire day ($T^c=1440$), however, they usually accumulate around lunch and dinner times. For that reason setting $T=T^c+120$ is sufficient in our experiments. Following the literature, we assume two demand peaks around lunch and dinner time. The number of lunch and dinner orders is given by random variables $OD$ and $OL$, each following a normal distribution. We assume that during lunch time, more orders are placed in the city center while during dinner time, orders are more frequent in residential areas. \rev{We provide detailed information on how we sample instances in Appendix~\ref{app:instance_generation}.}

Orders are equally likely for all restaurants. \rev{The preparation times of orders are drawn from a log-normal distribution mimicking a long tail behavior observed in practice \citep{mao2022demand}.} We assume decreasing expected preparation times over the five restaurants in steps of one minute from 10 to 6 minutes mimicking different cuisines or efficiency between the restaurants. We assume a delivery promise of $\tau=30$ minutes, a freshness constraint of $\delta_f=20$ minutes for all food types $f\in F$ and a capacity of the vehicles of $\kappa=3$ orders (we note that in our experiments, the capacity is rarely reached due to the freshness constraint).

We define three main instance settings. In the \emph{Small} instance setting, we assume $|C_j|=1$ cook for each restaurant $j$. We further assume $|V|=5$ vehicles. The expected number of orders during lunchtime is 64 and 100 for dinner time. For the \emph{Medium}  setting, we keep the resources the same, but increase demand by $25\%$. For the \emph{Large}  settings, we double the numbers of cooks and vehicles. We also double the demand compared to the Medium instances. For all settings, we create 300 days of order realizations for evaluation. 




\subsection{Tuning and Benchmark Methods}

We train our VFA on the Small instances over 10,000 simulations on realizations different than the 300 used for evaluation. We observe convergence after 4,000 simulations. In the training and implementation, we set the number of LNS-iterations to 70 for each state. We fine-tune the VFA each for the Medium and Large instances with additional 1,000 samples. 

We create four benchmark policies, two are problem-oriented and two are method-oriented. \rev{The latter can be found in Appendix~\ref{sec: analysis heuristic}.} The two problem-oriented benchmarks follow the concepts of \cite{rijal2023} and \cite{d2023integrated} for combined optimization of warehouse operations and vehicle routing, respectively. 

\begin{itemize}
    \item \emph{FIFO}: This policy implements the first-in-first-out starting solution in every state. It  therefore solves the cooking part first and then the routing part. As the FIFO solution is effective when considering the cooks in isolation and ignoring vehicles, it can be seen as the sequential approach proposed in \cite{rijal2023}.
    \item \emph{Integrated}: This policy considers cooking and routing together, however it does not anticipate future developments. It applies the same LNS as our policy, but searches for the decision that minimizes immediate delay.
\end{itemize}



\subsection{Policy Comparison}\label{sec: test heuristic} 

First, we compare our policy with the problem-oriented benchmarks. Table~\ref{tab:results basic instances} summarizes the results of our suggested solution method on the Small, Medium and Large instances compared to the two problem-oriented benchmark policies. The table lists (1) the instances type; (2) the key performance indication (KPI) examined \rev{(\enquote{late orders} refer to orders with a positive delay)}; (3)-(5) the tested policies; (6  \& 7) the average improvement of our \textit{AI} heuristic over \textit{FIFO} and \textit{Integrated}, respectively, for the corresponding instance type and KPI. All values are rounded to one digit but the improvement is calculated on the original values.

\rev{Depending on the instance and method, the average delay values range between 4.9 and 26.8 minutes which falls within ranges observed in practice \citep{mao2022demand}.}
We observe that our policy improves all KPIs for all three instance settings. For the Small instances, the average delay can be reduced substantially by $17.0\%$ compared to \textit{Integrated}, and even by $106.5\%$ compared to \textit{FIFO}. The same tendency can be observed for the Medium and Large instances even though the improvements are less extreme with $9.4\%$, and $68.1\%$ for the Medium Instances and $18.5\%$, and $42.2\%$ for the Large Instances.

\begin{table}[!t]
\caption{Average results for the Small, \rev{Medium} and Large instances}\label{tab:results basic instances}
\centering
\footnotesize
\begin{tabular}{llrrrrrrr}
\hline
Instances &
  KPI &
  \multicolumn{1}{l}{\textit{FIFO}} &
  \multicolumn{1}{l}{\textit{Integrated}} &
  \multicolumn{1}{l}{\textit{AI}} &
  \multicolumn{1}{l}{\begin{tabular}[l]{@{}l@{}}\% imp. over \\ \textit{FIFO}\end{tabular}} &
  \multicolumn{1}{l}{\begin{tabular}[l]{@{}l@{}}\% imp. over \\ \textit{Integrated}\end{tabular}} \\ \hline
\multirow{7}{*}{Small} &
  Avg. delay (in min.) &
  10.1 &
  5.8
  & 4.9
  & 106.5  
  &  17.0 
 \\ &
  Perc. of late orders (in \%) &
  43.4 &
36.4
  & 34.3
  &  26.4 
  & 6.1  
 \\ &
  Avg. delay of late orders (in min.) &
  22.2 &  14.7
  & 13.3
  &  66.8 
  &   10.8
 \\ &
  Max. delay  (in min.) &
  52.6 &
45.4
  & 44.0
  & 19.5  
  &   3.1
 \\ &
  Avg. click to door  (in min.) &
  34.7 &
 29.9
  & 29.3
  &  18.5 
  &  2.1 
 \\ &
  Avg. number of orders/trip &
   1.3 & 
 1.3 
   & 1.4
  &   6.7
  & 3.0  
 \\ &
 Avg. total travel time (in min.) &
 2543.6   & 
2476.5 & 
2416.8   & 
 5.2 &  
2.5  &   
 \\ \hline
\multirow{7}{*}{Medium} &
  Avg. delay (in min.) 
  & 26.8  
  &   17.5 
  & 15.9 
  &   68.1
  & 9.4  
  \\  
 & Perc. of late orders (in \%) &  60.8 
  &  54.9 
  &   54.5
  &  11.6 
  &   0.8
  \\  
 & Avg. delay of late orders (in min.)
  &   44.0
  &  31.5 
  &  29.0 
  &  51.9 
  &   8.9
  \\  
 & Max. delay  (in min.) 
 &   101.4 
  & 82.2  
  &  78.4
  &  29.3 
  & 4.8  
  \\ 
 & Avg. click to door  (in min.)
  &   53.2
  &  43.5 
  &  42.2
  &  26.2 
  &  3.2 
  \\ 
 & Avg. number of orders/trip 
  &   1.3
  & 1.4  
  &   1.4
  &   10.5
  & 4.2  
  \\ &
  Avg. total travel time (in min.) &
 3154.9   & 
3020.0 & 
 2940.5  & 
7.3  &  
2.7  &   
 \\ \hline
\multirow{7}{*}{\begin{tabular}[c]{@{}l@{}}Large 
\end{tabular}} &
  Avg. delay (in min.) 
  &  15.3 &
12.8  
  & 10.8
  &  42.2 
  &  18.5 
 \\ &
  Perc. of late orders (in \%) 
  &  46.8 &
 44.4
  & 44.1
  & 6.0  
  &  0.7 
 \\ &
  Avg. delay of late orders (in min.) 
  &  32.5 &
  28.4
  &24.0
  & 35.5  
  &   18.3
 \\ &
  Max. delay  (in min.) 
  &  69.8 &
  64.8
  &60.5
  &  15.3 
  &   7.1
 \\ &
  Avg. click to door  (in min.) 
  &  40.3 &
 37.7
  &36.0
  & 12.0  
  &  4.7 
 \\ &
  Avg. number of orders/trip 
  &  1.5 &
  1.5 & 
  1.5
  & 4.9  
  &  4.2 
 \\ &
Avg. total travel time (in min.) &
  5926.7  & 
5869.3 & 
5679.1   & 
 4.4 &  
3.4  &   
 \\ \hline
 \end{tabular}
\end{table}

The stepwise increase in solution quality from \textit{FIFO} to \textit{Integrated} to our policy confirms the effectiveness of both parts of our policy, the LNS-part for searching and the VFA-part for evaluating the decision space. The already large difference between \textit{FIFO} and \textit{Integrated} shows the value of joint optimization of scheduling and routing. The superiority of our policy for the Medium and Large instances also confirms the success of our transfer learning approach. In Appendix~\ref{app:A_results}, we further show that while the average delay depends on the net travel time from the ghost kitchen to the customer, our policy improves this KPI for customers across the city. It also shows that the average delay is linked to the restaurant's performance. Longer expected preparation times result in more delays for the respective customers of the restaurant.

While the objective of our problem and methodology is to minimize delay, we also observe significant improvements in other KPIs. Our method reduces the percentage of customers experiencing delay, the average delay, and the maximum delay a customer experiences per day. Thus, we have fewer dissatisfied customers and these customers are dissatisfied less. We also observe that the average click to door time decreases, thus, customers receive their food faster. This result is not self-evident as \cite{ulmer2021restaurant} showed that optimization for delay may trade fast deliveries for fewer late deliveries. Finally, we observe that our policy increases the average bundle size and reduces the average travel times by several percentages. 
Here, there is a significant difference between \textit{AI} and \textit{Integrated}. When looking closer at the decision making, we observed that the main differences in bundling between the two policies occur during the lunch and dinner peak times. That is, the anticipation of \textit{AI} allows preparing for the peaks and, consequently, leads to more effective decision making during the peaks. \textcolor{safeblue}{For a closer analysis of the workforce utilization, we refer to Appendix~\ref{app:A_util} where we show that \textit{AI} leads to a more effective utilization especially at the beginning of peaks and to a substantially smaller backlog at the end of the peaks.} 
In essence, the reduction in travel times and the increase in bundling opportunities does not only come from an elaborate search of the decision space, but also from the anticipation of future orders.

\color{safeblue}
\subsection{The Value of Ghost Kitchens}

Ghost kitchens are operated by the platform and are optimized for on-demand food delivery. The absence of dine-in customers enables the platform to set up ghost kitchens like industrial facilities in order to eliminate uncertainty in food preparation times. Since all restaurants share the same location, the potential of consolidated deliveries is increased. Further, the platform can schedule when and by which cook each order is prepared to fully synchronize the preparation and delivery processes. The preparation schedule is continuously adapted by the platform to react to new orders. In the following experiment, we quantify the impact of a ghost kitchen compared to relying on dine-in restaurants that participate as stakeholders in an on-demand meal delivery operation. 

The experiment is based on our previous three instances which we evaluated in Section~\ref{sec: test heuristic}. We use the same realization of demand including customer locations, order times, food type and preparation times. However, each food type is now prepared by a corresponding independent restaurant that is located within Iowa City. The locations of the restaurants are obtained by a $k$-medoid clustering of real restaurant locations in Iowa City, where $k$ is given by the number of food types. We add Gaussian noise (with a mean of 0 and a variance of 20\% of the expected preparation time) to the preparation times in order to represent that traditional restaurants are affected by disturbances in their preparation process. The restaurant's independence also implies that we cannot adjust the restaurants preparation schedule to synchronize between different restaurants (e.g., to enable bundling of orders from different restaurants). Therefore, each restaurant has one preparation queue for every given cook. Orders are assigned to the cook with the shortest queue and are prepared in first-in-first-out manner. The schedule cannot be modified ex-post. This implies that we cannot guarantee the freshness constraint and ignore it for this experiment, for a fair comparison, also for the ghost kitchen. The delivery process is handled by a fleet of vehicles operated by the platform. We use the assignment and routing policy of \cite{ulmer2021restaurant} combined with a repositioning strategy in which vehicles return to the closest restaurant if they are idle.

We summarize our results in Table~\ref{tab: results ghost vs multiple restaurants}. The table lists the instance type (1), the reported KPIs including the mean freshness of orders (2) for the setting of one ghost kitchen without freshness constraints using the \textit{AI} policy (3), for the same setting but using the \textit{FIFO} policy (4), and the setting of independent dine-in restaurants (5), as well as the ghost kitchen's (using the \textit{AI} policy) improvement on the setting of independent dine-in restaurants for each KPI (6).

\begin{table}[!t]
\color{safeblue}
\centering
\caption{Comparison of one ghost kitchen with multiple, independent restaurants}\label{tab: results ghost vs multiple restaurants}
\footnotesize
\begin{tabular}{llrrrr}
\hline
Instances &
  KPI &
  \multicolumn{1}{l}{\textit{\begin{tabular}[c]{@{}l@{}}Ghost Kitchen\\ (AI)\end{tabular}}} &
    \multicolumn{1}{l}{\textit{\begin{tabular}[c]{@{}l@{}}Ghost Kitchen\\ (FIFO)\end{tabular}}} &
  \multicolumn{1}{l}{\textit{\begin{tabular}[c]{@{}l@{}}Independent\\ Restaurants\end{tabular}}} &
  \multicolumn{1}{l}{\begin{tabular}[c]{@{}l@{}}\% imp. over\\ \textit{Independent Restaurants}\end{tabular}} \\ \hline
\multirow{8}{*}{Small}  & Avg. delay (in min.)                & 4.5 &  6.4 & 7.0    & 55.6 \\
                        & Perc. of late orders (in \%)        & 34.6 & 43.4 & 47.6   & 37.6 \\
                        & Avg. delay of late orders (in min.) & 12.0 & 13.8  & 14.6   & 21.7 \\
                        & Max. delay (in min.)                & 41.4 & 42.7  & 63.2   & 52.7 \\
                        & Avg. click to door (in min.)        & 29.1 & 31.2  & 32.6   & 12.0 \\            & Avg. freshness (in min.)            & 15.7  & 17.8 & 19.3   & 22.9 \\
                        & Avg. number of orders/trip          & 1.5  & 1.8 & 1.3    & 13.3 \\
                        & Avg. total travel time (in min.)    & 2389.2 & 2304.1 & 2672.5 & 11.9 \\ \hline
\multirow{8}{*}{Medium} & Avg. delay (in min.)                & 13.1 &  16.2 & 13.8   & 5.3  \\
                        & Perc. of late orders (in \%)        & 52.5 & 60.4  & 62.2   & 18.5 \\
                        & Avg. delay of late orders (in min.) & 24.2  & 26.2 & 22.1   & -8.7 \\
                        & Max. delay (in min.)                & 75.7  & 71.1 & 108.9  & 43.9 \\
                        & Avg. click to door (in min.)        & 39.3 & 42.7  & 40.8   & 3.8  \\
                        & Avg. freshness (in min.)            & 22.6  & 25.9 & 23.8   & 5.3 \\
                        & Avg. number of orders/trip          & 1.7   & 2.0 & 1.5    & 11.8 \\
                        & Avg. total travel time (in min.)    & 2902.8 & 2770.9 & 3014.3 & 3.8  \\ \hline
\multirow{8}{*}{Large}  & Avg. delay (in min.)                & 7.8  & 9.7  & 8.5    & 9.0  \\
                        & Perc. of late orders (in \%)        & 43.8 &  44.8 & 52.0   & 18.7 \\
                        & Avg. delay of late orders (in min.) & 17.4 &  21.2 & 16.3   & -6.3 \\
                        & Max. delay (in min.)                & 51.4  & 51.8 & 89.2   & 73.5 \\
                        & Avg. click to door (in min.)        & 33.4  & 34.9 & 34.7   & 3.9  \\
                        & Avg. freshness (in min.)            & 16.6   & 15.6 & 21.5   & 29.5 \\
                        & Avg. number of orders/trip          & 1.8  & 1.8 & 1.5    & 20.0 \\
                        & Avg. total travel time (in min.)    & 5417.3 & 5540.2  & 6057.8 & 11.8 \\ \hline
\end{tabular}
\end{table}

We observe that a ghost kitchen benefits customers and vehicle drivers. It benefits customers because it yields a higher service quality than a setting with independent dine-in restaurants. This manifests in a lower average delay and fresher food, but most prominently, in a drastically lower percentage of delayed deliveries as well as in a significantly lower (mean) maximum delay over the simulated days. 

The average delay of late orders is the the only KPI that worsens (for the Medium and Large instances). However, this might be a statistical artifact due to having fewer delayed orders.
Furthermore, a ghost kitchen benefits drivers as it reduce their daily average travel times (in minutes) by increasing the number of bundled orders compared to the setting of independent restaurants. The increase in bundles can be expected due to the single restaurant location and the optimization of food preparation. The savings in travel time usually also benefit the platform, dependent on the applied compensation scheme. We further observe that the described benefits are only achieved if the ghost kitchen is operated in an efficiently and anticipative manner. When using the \textit{FIFO} policy to schedule the meal preparations and vehicle routes, we observe higher mean delays than for the case of independent restaurants in both the medium and large instances.


\color{black}
 \subsection{Problem Analysis} \label{sec: problem analysis}

 
 In this section, we perform a sensitivity analysis and study the value of collaboration. In our main experiments for the Large setting, we assume a freshness constraint of $\delta_f=20$ minutes, a delivery promise of $\tau=30$ minutes, $|C_j|=2$ cooks per restaurant, $|V|=10$ vehicles, moderate standard deviations $\sigma(t^p_i\mid f_i=f)$ in preparation times for all customers $i\in I$ and food types $f\in F$, and expected demand values of $\mathbb{E}[OL]=160$ and $\mathbb{E}[OD]=250$. We now vary the parameters individually and analyze their impact on the KPIs. The results of our \textit{AI} policy are shown in Table~\ref{tab: Parameter analysis}.

 \paragraph{Freshness constraint.} The first two instances, L1 and L2, increase and decrease the freshness constraint by five minutes, respectively. We observe that the average delay differs significantly between the two instances. One reason for the difference is that with a more relaxed freshness constraint, the number of orders per trip can be increased significantly (1.2 to 1.5. and even 1.8). This consolidation frees vehicle resources and allows for faster deliveries. This illustrates a trade-off for restaurants and ghost kitchens. If customers want even fresher food, they likely pay for it by longer waiting. We further note that average and maximum delay values are also indicators for overall working times for cooks and vehicles. Thus, to ensure fresh food, the platform likely has to pay their employees longer.

\begin{table}[!t]
\caption{Sensitivity Analysis}\label{tab: Parameter analysis}
\centering
\small
\begin{tabular}{llccrrrrr}
\hline
\multirow{2}{*}{Instance} &
  \multirow{2}{*}{Parameter} &
  \multicolumn{1}{l}{\multirow{2}{*}{\begin{tabular}[c]{@{}l@{}}Original \\ Value\end{tabular}}}&
    \multicolumn{1}{l}{\multirow{2}{*}{\begin{tabular}[c]{@{}l@{}}Updated \\ Value\end{tabular}}}&
  \multicolumn{3}{c}{Delay (in min.)} &
  \multicolumn{1}{l}{\multirow{2}{*}{\begin{tabular}[c]{@{}l@{}}\rev{\% of} \\ \rev{late orders}\end{tabular}}} &
  \multicolumn{1}{l}{\multirow{2}{*}{\begin{tabular}[c]{@{}l@{}}\rev{Avg. number} \\ \rev{of orders/trip}\end{tabular}}} \\
 &
   &
   &
   &
  \multicolumn{1}{l}{Avg.} &
  \multicolumn{1}{l}{Max.} &
  \multicolumn{1}{l}{Avg. late orders} &
  \multicolumn{1}{l}{} &
  \multicolumn{1}{l}{} \\ \hline
L1    & $\delta_f,  \forall f\in F$ &10& 15         & 24.1 &99.5 &45.6 & 53.1 & 1.2 \\
\smallskip L2    & $\delta_f,  \forall f\in F$ &10& 25 & 7.8 & 51.4 & 17.4 & 43.8 & 1.8 \\
L3    & $\tau$ &30 & 25 & 12.9 & 64.8 & 23.3 & 55.0 & 1.5 \\ 
\smallskip L4    & $\tau$ &30  & 35 &8.7 & 55.0 & 22.4 &37.4 & 1.5 \\
L5    & $|C_j|,  \forall j\in J$ &2& 1 & 51.7 & 262.5 & 71.0 &72.8 & 1.3 \\
\smallskip L6    & $|C_j|,  \forall j\in J$ &2 & 3 &10.6 & 60.0 & 26.0 &39.7 & 1.5 \\
L7    & $|V|$ &10 & 7 & 53.6 & 190.0 & 74.8 & 71.7 & 1.5 \\
\smallskip L8   & $|V|$ &10 & 13 & 3.1 & 41.4 & 11.4 & 25.1 & 1.5 \\
L9    & $\sigma(t^p_i)$&$\sigma(t^p_i)$& 0 & 9.9 &54.4 &23.9 &40.7 & 1.6 \\
\smallskip L10   & $\sigma(t^p_i)$&$\sigma(t^p_i)$&$2\cdot\sigma(t^p_i)$&18.4 & 119.2 & 34.4 &53.2 &1.5 \\
L11   & $\mathbb{E}[OL]$, $\mathbb{E}[OD]$ & 160,250 &144, 225 & 6.4 & 44.8 & 16.7 & 36.4 &1.5 \\
L12   & $\mathbb{E}[OL]$, $\mathbb{E}[OD]$ & 160,250 & 176, 275 & 19.9 & 88.2 & 38.0 & 52.2 & 1.5 \\
\hline
\rev{\textbf{Large}} & &  & & \rev{\textbf{10.8}} & \rev{\textbf{60.5}} & \rev{\textbf{24.0}} & \rev{\textbf{44.1}} & \rev{\textbf{1.5}} \\ \hline

\end{tabular}
\vspace{-2mm}
\end{table}

\paragraph{Delivery promise.} Instances L3 and L4 increase and decrease the delivery promise by 5 minutes, respectively. The difference in delivery promises between the two instances is almost equal to the difference in observed maximum delay (9.8 minutes) for the two instances. The difference in average delay is about 6.4 minutes, thus, smaller than 10 minutes, because in L4 some deliveries were made earlier than the deadline.

 \paragraph{Resources.} Instances L5 to L8 vary the number of cooks (L5, L6) and the number of vehicles (L7, L8). The results indicate, as expected, that reducing any resource leads to a tight bottleneck and unacceptably high delays while increasing any resources improves the service quality. Notably, when the number of cooks are reduced in L5, the number of orders/trip decrease as well, from 1.5 to 1.3. 
While bundling increases efficiency in the vehicle utilization, it likely reduces efficiency for the cooks due to the synchronization requirements for orders that share trips. Our method recognizes that cooks are the bottleneck and dispatches vehicles directly more often instead of enforcing consolidation.
 In essence, the ghost kitchen provider should carefully balance the two resources required for operations as a lack in one can hardly be compensated by the other.

\paragraph{Preparation times.} Instances L9 and L10 vary the standard deviation of preparation times, either by setting it to zero or by doubling it. We observe that the value of no variance in preparation time (9.9 average delay) is only slightly smaller compared to our standard instance (10.8 average delay). \rev{This is remarkable since this indicates that the long tail values of the preparation times are not the main reason for longer delays. Rather, delay results from the complex  dynamics of our problem.} When the variance is large, the system collapses (40.0 average delay). This indicates that the system can handle slight variances in preparation times, however, it is very important for the ghost kitchen provider to keep variations manageable. 

\paragraph{Demand.} Instances L11 and L12 decrease and increase demand by 10\%. As expected with less demand, average delay decreases and with more demand, average delay becomes significantly worse. Thus, accurate demand predictions are crucial to manage resources accordingly.

\begin{table}[!t]
\caption{The value of collaborating cooks and vehicles.}\label{tab: collaboration_cooks}
\centering
\small
\begin{tabular}{lrrr r rr}
\hline
                                    & \multicolumn{3}{c}{Cooks} & &\multicolumn{2}{c}{Vehicles} \\
KPI & \multicolumn{1}{l}{No} & \multicolumn{1}{l}{Pairwise} & \multicolumn{1}{l}{Full} & &\multicolumn{1}{l}{Pairwise} & \multicolumn{1}{l}{Full} \\ \hline
Avg. delay (in min.)                & \textcolor{black}{16.7}    & \textcolor{black}{10.8}   & \textcolor{black}{10.0}  &  & \textcolor{black}{34.3}         & \textcolor{black}{10.8}       \\
Percentage of late orders (in \%)        & \textcolor{black}{51.5}    & \textcolor{black}{44.1}   & \textcolor{black}{38.2} &  & \textcolor{black}{51.8}         & \textcolor{black}{44.1}       \\
Avg. delay of late orders (in min.) & \textcolor{black}{32.2}    & \textcolor{black}{24.0}   & \textcolor{black}{25.2} &  & \textcolor{black}{65.8}         & \textcolor{black}{24.0}      \\
Avg. click to door (in min.)        & \textcolor{black}{42.8}    & \textcolor{black}{36.0}   & \textcolor{black}{34.4}  & & \textcolor{black}{}60.5         & \textcolor{black}{36.0}       \\
Max. delay  (in min.)               & \textcolor{black}{89.7}    & \textcolor{black}{60.5}   & \textcolor{black}{54.6} &  & \textcolor{black}{181.8}        & \textcolor{black}{60.5}       \\
Avg. number of orders/trip          & \textcolor{black}{1.5}     & \textcolor{black}{1.5}    & \textcolor{black}{1.5}   & & \textcolor{black}{1.3}          & \textcolor{black}{1.5}       \\ \hline

\end{tabular}
\end{table}

Finally, we analyze the value of shared resources between the restaurants. In our initial setting, we assumed full collaboration amongst the vehicle fleet and pairwise collaboration of two cooks per restaurant. We now vary the collaborations between no, pairwise, and full for the two types of resources individually. The results are shown in Table~\ref{tab: collaboration_cooks}. No collaboration between cooks means that instead of five food types with two cooks each, we assume ten food types with single cooks. We observe that in this case, delay increases. Thus, even the pairwise collaboration is already very beneficial. Delay reduces even further, but only slightly, for the full collaboration case where every cook can prepare every food (essentially one giant restaurant). The limited improvement is likely because the vehicles become the main bottleneck in that case. For vehicles, we observe a similar, but more severe result. Delay goes up and orders per trip go down. Having dedicated vehicles per restaurant is undesirable as one major advantage of the ghost kitchen is lost, the central, anticipatory preparation scheduling and joint delivery for different restaurants. This experiment reveals two main insights: First, a careful balance of the resources is required for effective operations. Second, the central dispatching of a joint fleet is one of the major advantages of a ghost kitchen.

\color{safeblue}
\section{Discussion} \label{sec: Discussion}
\rev{In this paper, we presented a novel solution method to synchronize between the first (preparation) and second (dispatching) stages of a ghost kitchen's operation. Our solution method consists of several components whose usefulness}
are not limited to the RMD-GK, but can also be used in other problems that have a large decision space,
and/or real time decision making under uncertainty.
For instance, the notion of a condensed decision representation can be used to reduce the search space (ideally without loss of optimality) in problems in which homogeneous resources are considered, e.g., vehicle fleet management \rev{and parallel machines scheduling}.
Similarly, searching over partial decisions and determining timing by a separate algorithm can be a successful approach for various scheduling and assignment problems.
This is the case when the starting times of tasks can be defined as a sub problem of the scheduling problem, which may be easier to solve given the assignment of tasks at the resources. This motivates searching over assignment decisions only and determining the starting times according to a separate algorithm.
\rev{Finally}, our methodology can be used in other complex problems \rev{in which two sequential stages need to be coordinated}, for example, when facing joint dynamic optimization of warehouse operations and vehicle routing \citep{rijal2023}. \rev{In such cases, infeasibility  often arises due to the need to synchronize the stages. Analytical analysis that excludes from consideration a significant number of infeasible solutions may greatly contribute to the solution method's efficiency, 
as demonstrated for the RMD-GK.
Our numerical study reveals that synchronization achieves a significant improvement in all important KPIs and that interesting and insightful connections exist between the availability and use of the resources allocated to the two stages. This is, again, a promising direction for future research for other problems with similar characteristics.}

\color{black} 
\section{Conclusion \rev{and Future Work}} \label{sec: conclusion}

In this work, we have analyzed the concept of ghost kitchens and have shown how anticipatory and integrated optimization of scheduling and routing can improve operations and customer experience. There are several avenues for future research in problem and method.
\rev{In our work, we have shown that operating a ghost kitchen has several advantages compared to conventional meal delivery systems by analyzing both systems in isolation, each with a dedicated fleet. Future work may investigate hybrid systems of both regular restaurants and ghost kitchens or how ghost kitchens can be implemented with crowdsourced delivery services.}
Further, we have assumed that customers still order from one restaurant only. However, an additional advantage of ghost kitchens might be that customers may order from different restaurants and receive their food at the same time \citep{arslan2021operational}. 
Future research may extend the presented model and methodology to allow for multiple orders and analyze its impact on the system's performance, since such an option may lead to less flexibility in the ghost kitchen and longer delivery times. 
We have further shown that the freshness constraint has a significant impact on the delivery times. Future research may exploit this insight, for example, by dynamically relaxing the constraint during peak times. Our results also indicate that the collaboration amongst cooks can be beneficial. 
The setup and assignment of cooks may deserve a closer investigation in the future. For example, some \enquote{allround} cooks may be assigned dynamically between the restaurants (or even to vehicles) based on realized and expected demand. 
We have further seen that the variance in preparation times of the restaurants determine success or failure of the entire operation. Future work may investigate this in more detail, e.g., by analyzing how a single restaurant's performance impacts the service quality of the other restaurants or by modeling preparation time uncertainty. \rev{Moreover, exploring the fairness between customers, cooks, and vehicles could be an intriguing avenue for future research, potentially leading to an even more balanced and equitable system-wide performance.}

In our work, we have presented a method that searches and evaluates complex scheduling and routing decisions in an integrated fashion. There are a variety of possible extensions. In our method, we have condensed the decision space to search by using a condensed partial decision representation, and if feasible, created the full decision for evaluation via our VFA. 
In the future, VFA-features may be developed to allow evaluation already on the condensed partial decision level. 
While our PDFT-algorithm was able to identify feasibility in the first iterations, future research may combine our analytical checks with checks based on machine learning \citep{van2022machine}. Another extension could be to use the trained neural net for guiding the search within the LNS, e.g., by prescribing the operators to use in a state.

\section*{Acknowledgments}
We would like to thank Julius Hoffmann for his valuable advice on the literature on stochastic dynamic customer order scheduling. Gal Neria's research is partially supported by the Israeli Smart Transportation Research Center (ISTRC) and the Council for Higher Education in Israel (VATAT). Florentin Hildebrandt's research is funded by the Deutsche Forschungsgemeinschaft (DFG, German Research Foundation), project 413322447.  Michal Tzur’s research is partially supported by the Israeli Smart Transportation Research Center (ISTRC). Marlin Ulmer's work is funded by the DFG Emmy Noether Programme, project 444657906. We gratefully acknowledge their support.

\bibliographystyle{apalike}  
\bibliography{main}

\begin{thebibliography}{}

\bibitem[Arslan et~al., 2021]{arslan2021operational}
Arslan, A.~M., Agatz, N., and Klapp, M.~A. (2021).
\newblock Operational strategies for on-demand personal shopper services.
\newblock {\em Transportation Research Part C: Emerging Technologies}, 130:103320.

\bibitem[Auad et~al., 2023]{auad2023courier}
Auad, R., Erera, A., and Savelsbergh, M. (2023).
\newblock Courier satisfaction in rapid delivery systems using dynamic operating regions.
\newblock {\em Omega}, page 102917.

\bibitem[Auad et~al., 2024]{auad2022dynamic}
Auad, R., Erera, A., and Savelsbergh, M. (2024).
\newblock Dynamic courier capacity acquisition in rapid delivery systems: A deep q-learning approach.
\newblock {\em Transportation Science}, 58(1):67--93.

\bibitem[Boeing, 2017]{boeing2017osmnx}
Boeing, G. (2017).
\newblock {OSMnx}: New methods for acquiring, constructing, analyzing, and visualizing complex street networks.
\newblock {\em Computers, Environment and Urban Systems}, 65:126--139.

\bibitem[Creswell, 2024]{ghostkitchens2024}
Creswell, J. (2024).
\newblock Ghost kitchens are disappearing, squeezed by demand and complaints.
\newblock \url{https://www.nytimes.com/2024/04/12/business/ghost-kitchens-restaurants-pandemic.html}.
\newblock Online; accessed 02 July 2024.

\bibitem[Dai and Liu, 2020]{dai2020workforce}
Dai, H. and Liu, P. (2020).
\newblock Workforce planning for {O2O} delivery systems with crowdsourced drivers.
\newblock {\em Annals of Operations Research}, 291(1):219--245.

\bibitem[D'Haen et~al., 2023]{d2023integrated}
D'Haen, R., Braekers, K., and Ramaekers, K. (2023).
\newblock Integrated scheduling of order picking operations under dynamic order arrivals.
\newblock {\em International Journal of Production Research}, 61(10):3205--3226.

\bibitem[Drexl, 2012]{drexl2012synchronization}
Drexl, M. (2012).
\newblock \color{safeblue} {Synchronization} in vehicle routing—a survey of {VRPs} with multiple synchronization constraints.
\newblock {\em Transportation Science}, 46(3):297--316.

\bibitem[Feldman et~al., 2023]{feldman2022managing}
Feldman, P., Frazelle, A.~E., and Swinney, R. (2023).
\newblock Managing relationships between restaurants and food delivery platforms: Conflict, contracts, and coordination.
\newblock {\em Management Science}, 69(2):812--823.

\bibitem[He et~al., 2015]{he2015delving}
He, K., Zhang, X., Ren, S., and Sun, J. (2015).
\newblock Delving deep into rectifiers: Surpassing human-level performance on imagenet classification.
\newblock In {\em Proceedings of the IEEE International Conference on Computer Vision}, pages 1026--1034.

\bibitem[Heinold et~al., 2023]{heinold2022primal}
Heinold, A., Meisel, F., and Ulmer, M.~W. (2023).
\newblock Primal-dual value function approximation for stochastic dynamic intermodal transportation with eco-labels.
\newblock {\em Transportation Science}, 57(6):1452--1472.

\bibitem[Hildebrandt et~al., 2023]{HILDEBRANDT2023106071}
Hildebrandt, F.~D., Thomas, B.~W., and Ulmer, M.~W. (2023).
\newblock Opportunities for reinforcement learning in stochastic dynamic vehicle routing.
\newblock {\em Computers \& Operations Research}, 150:106071.

\bibitem[Hildebrandt and Ulmer, 2022]{hildebrandt2020supervised}
Hildebrandt, F.~D. and Ulmer, M.~W. (2022).
\newblock Supervised learning for arrival time estimations in restaurant meal delivery.
\newblock {\em Transportation Science}, 56(4):1058--1084.

\bibitem[Hossein Nia~Shavaki and Jolai, 2021]{hossein2021rule}
Hossein Nia~Shavaki, F. and Jolai, F. (2021).
\newblock A rule-based heuristic algorithm for joint order batching and delivery planning of online retailers with multiple order pickers.
\newblock {\em Applied Intelligence}, 51(6):3917--3935.

\bibitem[Jahanshahi et~al., 2022]{jahanshahi2022deep}
Jahanshahi, H., Bozanta, A., Cevik, M., Kavuk, E.~M., Tosun, A., Sonuc, S.~B., Kosucu, B., and Ba{\c{s}}ar, A. (2022).
\newblock A deep reinforcement learning approach for the meal delivery problem.
\newblock {\em Knowledge-Based Systems}, 243:108489.

\bibitem[Kingma and Ba, 2015]{kingma2015adam}
Kingma, D. and Ba, J.~L. (2015).
\newblock Adam: A method for stochastic optimization.
\newblock In {\em 3rd International Conference on Learning Representations, Date: May 7-9, Location: San Diego, CA, USA}.

\bibitem[Klapp et~al., 2018]{klapp2018dynamic}
Klapp, M.~A., Erera, A.~L., and Toriello, A. (2018).
\newblock The dynamic dispatch waves problem for same-day delivery.
\newblock {\em European Journal of Operational Research}, 271(2):519--534.

\bibitem[Lin, 1992]{lin1992self}
Lin, L.-J. (1992).
\newblock Self-improving reactive agents based on reinforcement learning, planning and teaching.
\newblock {\em Machine Learning}, 8:293--321.

\bibitem[Liu et~al., 2022]{liu2022approximate}
Liu, H., Wang, Y., Lee, L.~H., and Chew, E.~P. (2022).
\newblock An approximate dynamic programming approach for production-delivery scheduling under non-stationary demand.
\newblock {\em Naval Research Logistics (NRL)}, 69(4):511--528.

\bibitem[Liu, 2019]{liu2019optimization}
Liu, Y. (2019).
\newblock An optimization-driven dynamic vehicle routing algorithm for on-demand meal delivery using drones.
\newblock {\em Computers \& Operations Research}, 111:1--20.

\bibitem[Mao et~al., 2022]{mao2022demand}
Mao, W., Ming, L., Rong, Y., Tang, C.~S., and Zheng, H. (2022).
\newblock On-demand meal delivery platforms: Operational level data and research opportunities.
\newblock {\em Manufacturing \& Service Operations Management}, 24(5):2535--2542.

\bibitem[Moons et~al., 2017]{moons2017integrating}
Moons, S., Ramaekers, K., Caris, A., and Arda, Y. (2017).
\newblock Integrating production scheduling and vehicle routing decisions at the operational decision level: a review and discussion.
\newblock {\em Computers \& Industrial Engineering}, 104:224--245.

\bibitem[Neria and Tzur, 2024]{NeriaTzur2022}
Neria, G. and Tzur, M. (2024).
\newblock The dynamic pickup and allocation with fairness problem.
\newblock {\em \color{safeblue}Transportation Science,}.

\bibitem[Reyes et~al., 2018]{reyes2018meal}
Reyes, D., Erera, A., Savelsbergh, M., Sahasrabudhe, S., and O’Neil, R. (2018).
\newblock The meal delivery routing problem.
\newblock Technical report, Georgia Institute of Technology, Atlanta, GA.

\bibitem[Rijal et~al., 2023]{rijal2023}
Rijal, A., Bijvank, M., and de~Koster, R. (2023).
\newblock Dynamics between warehouse operations and vehicle routing.
\newblock {\em Production and Operations Management}, 32(11):3575--3593.

\bibitem[Rivera and Mes, 2017]{rivera2017anticipatory}
Rivera, A. E.~P. and Mes, M.~R. (2017).
\newblock Anticipatory freight selection in intermodal long-haul round-trips.
\newblock {\em Transportation Research Part E: Logistics and Transportation Review}, 105:176--194.

\bibitem[Rivera and Mes, 2022]{rivera2022anticipatory}
Rivera, A. E.~P. and Mes, M.~R. (2022).
\newblock Anticipatory scheduling of synchromodal transport using approximate dynamic programming.
\newblock {\em Annals of Operations Research}.

\bibitem[Shapiro, 2023]{shapiro2023platform}
Shapiro, A. (2023).
\newblock Platform urbanism in a pandemic: Dark stores, ghost kitchens, and the logistical-urban frontier.
\newblock {\em Journal of Consumer Culture}, 23(1):168--187.

\bibitem[Silva et~al., 2023]{SILVA2023100105}
Silva, M., Pedroso, J.~P., and Viana, A. (2023).
\newblock Deep reinforcement learning for stochastic last-mile delivery with crowdshipping.
\newblock {\em EURO Journal on Transportation and Logistics}, 12:100105.

\bibitem[Steever et~al., 2019]{steever2019dynamic}
Steever, Z., Karwan, M., and Murray, C. (2019).
\newblock Dynamic courier routing for a food delivery service.
\newblock {\em Computers \& Operations Research}, 107:173--188.

\bibitem[Ulmer et~al., 2022]{ulmer2022dynamic}
Ulmer, M.~W., Erera, A., and Savelsbergh, M. (2022).
\newblock Dynamic service area sizing in urban delivery.
\newblock {\em OR Spectrum}, 44(3):763--793.

\bibitem[Ulmer et~al., 2020]{ulmer2020modeling}
Ulmer, M.~W., Goodson, J.~C., Mattfeld, D.~C., and Thomas, B.~W. (2020).
\newblock On modeling stochastic dynamic vehicle routing problems.
\newblock {\em EURO Journal on Transportation and Logistics}, 9(2):100008.

\bibitem[Ulmer and Savelsbergh, 2020]{ulmer2020workforce}
Ulmer, M.~W. and Savelsbergh, M. (2020).
\newblock Workforce scheduling in the era of crowdsourced delivery.
\newblock {\em Transportation Science}, 54(4):1113--1133.

\bibitem[Ulmer et~al., 2021]{ulmer2021restaurant}
Ulmer, M.~W., Thomas, B.~W., Campbell, A.~M., and Woyak, N. (2021).
\newblock The restaurant meal delivery problem: Dynamic pickup and delivery with deadlines and random ready times.
\newblock {\em Transportation Science}, 55(1):75--100.

\bibitem[van~der Hagen et~al., 2022]{van2022machine}
van~der Hagen, L., Agatz, N., Spliet, R., Visser, T.~R., and Kok, L. (2022).
\newblock Machine learning–based feasibility checks for dynamic time slot management.
\newblock {\em Transportation Science}.

\bibitem[Xu et~al., 2016]{xu2016stochastic}
Xu, X., Zhao, Y., Wu, M., Zhou, Z., Ma, Y., and Liu, Y. (2016).
\newblock Stochastic customer order scheduling to minimize long-run expected order cycle time.
\newblock {\em Annals of Operations Research}.

\bibitem[Yildiz and Savelsbergh, 2019]{yildiz2019provably}
Yildiz, B. and Savelsbergh, M. (2019).
\newblock Provably high-quality solutions for the meal delivery routing problem.
\newblock {\em Transportation Science}, 53(5):1372--1388.

\bibitem[Zhang et~al., 2019]{zhang2019online}
Zhang, J., Liu, F., Tang, J., and Li, Y. (2019).
\newblock The online integrated order picking and delivery considering pickers’ learning effects for an {O2O} community supermarket.
\newblock {\em Transportation Research Part E: Logistics and Transportation Review}, 123:180--199.

\bibitem[Zhao et~al., 2018a]{zhao2018minimizing}
Zhao, Y., Xu, X., and Li, H. (2018a).
\newblock Minimizing expected cycle time of stochastic customer orders through bounded multi-fidelity simulations.
\newblock {\em IEEE Transactions on Automation Science and Engineering}, 15(4):1797--1809.

\bibitem[Zhao et~al., 2018b]{zhao2018stochastic}
Zhao, Y., Xu, X., Li, H., and Liu, Y. (2018b).
\newblock Stochastic customer order scheduling with setup times to minimize expected cycle time.
\newblock {\em International Journal of Production Research}, 56(7):2684--2706.

\end{thebibliography}

\FloatBarrier
\newpage
\begin{appendices}

\setcounter{table}{0}
\setcounter{section}{0}
\setcounter{equation}{0}
\setcounter{figure}{0}
\setcounter{lemma}{0}
\renewcommand{\thetable}{A\arabic{table}}
\renewcommand{\theequation}{A\arabic{equation}}
\renewcommand{\thesubsection}{A.\arabic{subsection}}
\renewcommand{\thesection}{A.\arabic{section}}
\renewcommand{\thefigure}{A\arabic{figure}}

\section{Decision Space}\label{app: decision space}
In the following, we formally define the decision space for a state $S_k$ with the help of a mixed integer linear program. We define the auxiliary decision variables and constraints of the MILP before we explain how the actual decision $x_k\in\mathcal{X}(S_k)$ can be recovered from a solution to the MILP.\\
We define the following decision variables.
    Let $z_{ict}\in\{0,1\}$ denote whether cook $c\in C_{f}$ starts preparing 
    order $i\in I_{kf}$ in time period $t_k\leq t\leq T, t\in\mathbb{N}$.
    Let $y_{ivt}\in\{0,1\}$ denote whether vehicle $v\in V$ departs the ghost kitchen at time $t$ to deliver order $i\in I_k$. Let $t^a_{i}\in[t_k,T]$ denote the arrival time at the destination of order $i\in I_k$. We define $t^r_{vt}\in[0,T]$ as the return time of vehicle $v$ which was dispatched in time $t$ ($t^r_{vt}$ has no significance if vehicle $v$ does not dispatch in time $t$). 
    Let $q_{ijt}\in\{0,1\}$ specify whether order $i\in I_k$ is sharing a trip that departs at $t\in[t_k, T]$ with order ${j\in I_k}$ and $j$ is the subsequent stop in the trip after $i$.     
Then, the corresponding decision space is given by 
\begin{alignat}{3}
&\displaystyle \sum_{v\in V}\sum_{t=t^o_i}^T y_{ivt} = 1 & i\in I_k \label{con:1}\\
&\displaystyle \sum_{c\in C_f} \sum_{t=t^o_i}^T z_{ict} = 1 & i\in I_k \label{con:2}\\
&\displaystyle 
\sum_{i\in I_{kf}}  z_{ict}  
\leq 1 & f \in F, c \in C_f , t_k\leq t\leq T  \label{con:3}\\
&\displaystyle 
\sum_{c\in C_f}\sum_{t=1}^T t\cdot z_{ict} + t^p_i\cdot z_{ict}
 \leq \sum_{v\in V}\sum_{t=1}^T t\cdot y_{ivt} 
& i\in I_k  \label{con:4}\\
&\displaystyle 
t^a_i - \sum_{c\in C_f}\sum_{t=1}^T (t+ t^p_i)\cdot z_{ict} \leq \delta_f
& f\in F, i\in I_{kf}  \label{con:5}\\
&\displaystyle (y_{ivt}-1)\cdot T +  t^a_i + t^t_{i0} \leq t^r_{vt} 
& i\in I_k,v \in V, t_k \leq t \leq T 
  \label{con:6}\\
&\displaystyle 
t^r_{vt^\prime} 
\leq
y_{ivt}\cdot t +T(1-y_{ivt})
& i \in I_k, v \in V, t_k \leq t^\prime < t \leq T
  \label{con:7}\\
&\displaystyle 
\sum_{i\in I_k} y_{ivt} 
\leq \kappa
& v \in V, t_k\leq t \leq T  \label{con:8}\\
& \sum_{i\in I_k} y_{ivt} - 1 \leq \displaystyle \sum_{i\in I_k}\sum_{\substack{j \in I_k\\ j\neq i}} q_{jit}
& v \in V, t_k\leq t \leq T \label{con:9} \\
&\displaystyle 2\cdot q_{j i t}  \leq y_{ivt} +y_{jvt} & i,j\in I_k,v \in V, t_k\leq t \leq T  \label{con:10}
\\
&\displaystyle  ({q}_{ijt} -1)\cdot T +t^t_{ij} \leq t^a_j  - t^a_i
& i,j\in I_k, i\neq j, t_k\leq t \leq T  \label{con:11} \\
&\displaystyle t\cdot y_{ivt} + t^t_{0i} \leq t^a_i & i\in I_k,v \in V, t_k\leq t\leq T  \label{con:12} \\
& \displaystyle 
\sum_{j\in I_{kf}}
({t^\prime} + t^p_j) z_{jc{t^\prime}} 
\leq
T + z_{ict}(t -T)
& f\in F, i\in I_{kf}, c \in C_f, t_k\leq t^\prime < t \leq T
  \label{con:13}\\
&\displaystyle z_{ict} = 0 & i\in I_k, c \in C, t_k\leq t \leq t^o_i  \label{con:14}\\
&\displaystyle y_{ivt} = 0  & i\in I_k,v \in V, t^o_{ik}\leq t\leq t^o_i \label{con:15}\\
&\displaystyle z_{ict}=1 & c\in C, i\in \psi_{kc}, t=t^s_{ik}\leq t_k \label{con:16}\\
&\displaystyle z_{ict} \in \{0,1\} & i\in I_k, c \in C, t_k \leq t \leq T  \label{con:17}\\
&\displaystyle y_{ivt} \in \{0,1\}  & 
i\in I_k,v \in V, t^r_{kv}\leq t \leq T \label{con:18}
\\
&\displaystyle q_{ijt} \in \{0,1\}  
& i,j\in I_k, t_k\leq t \leq T \label{con:19} \\
&\displaystyle t^a_i \in [t_k,T]  
& i\in I_k \label{con:20} \\
&\displaystyle t^r_{vt} \in [t_k,T]  
& v\in V, t\in[t_k,T] \label{con:21}
\end{alignat}

Constraints \eqref{con:1} assert that each order is delivered.
Constraints \eqref{con:2} assert that each order is prepared. 
Constraints \eqref{con:3} assert that each cook can only start producing one order at any time. 
Constraints \eqref{con:4} assert that each order can only be delivered after it has been produced.
Constraints \eqref{con:5} assert that each order must be delivered within the food type's freshness span after it has been prepared. 
Constraints \eqref{con:6} define the return time of each vehicle dispatched at time $t$. Then, Constraints \rev{\eqref{con:7}} assert that subsequent trips must occur later than its return time. 
Constraints \eqref{con:8} assert that the capacity of each vehicle is not violated.
Constraints \eqref{con:9} and \eqref{con:10} fix the sequence of orders within a trip. 
Constraints \eqref{con:11} and \eqref{con:12} set order arrival times and eliminate sub tours. 
Constraints \eqref{con:13} define when a cook is available to prepare the next meal. 
Constraints \eqref{con:14} - \eqref{con:15} set variables to zero if the departure time or preparation time is before the the corresponding order time.
Constraints \eqref{con:16} fix cook assignments and times if the preparation has already started.
Constraints \eqref{con:17} - \eqref{con:21} fix the domains of the variables.

Let the tuple of decision variables $(y, z, q, a)$ be a feasible solution to the MILP. Then, we may recover a decision $x_k=(\Psi_k^x, t_k^{s,x}, \Theta_k^x, t^{d,x}_k)$ as follows.
For a given cook $c\in C$, let $t_c=(t_{cn})_{n\in\mathbb{N}}$ be an ordered sequence of time points with $\sum_{i\in I_k}z_{ict}=1$ for every $t\in t_c$. 
Further, let $i_c=(i_{n}^c)_{n\in\mathbb{N}}$
be the corresponding sequence of orders such that $z_{ictk}=1$ for $i\in i_c, t\in t_c$ and $i,t$ being at the same position in their respective sequences. Then, the cook's schedule is given by $\psi_{kc}^x=i_c$. 
Given all cook schedules, the start of preparation times for each order $i\in I_k$, is given by $t^{s,x}_{ik}=t_{cn}$ with $i_{n}^c=i$. 
This is well-defined as each order is prepared by exactly one cook and only one order is prepared at each given time. We recover vehicle schedules in a similar manner. For a given vehicle $v\in V$, let $t_v=(t_{vn})_{n\in\mathbb{N}}$ be an ordered sequence of time points with $\sum_{i\in I_k}y_{ivt}=1$ for every $t\in t_v$. Then, let $t\in t_c$ be the $m$th time point in the sequence and let $i_{vt}=(i_{vtn})_{n\in\mathbb{N}}$ be the corresponding sequence of orders such that $y_{i_n,c,t,k}=1$ for 
$i_{n}^c\in i_c$
and $q_{i_n, i_{n+1}}=1$. Then, $\theta_{kvm}=i_{vt}$ is the $m$th trip in $\Theta_{kv}^x$ with departure time $t^{d,x}_{k\theta}=t$. Finally, we obtain $\Theta_k^x$ and $t_k^{d,x}$ by repeating the process for all vehicles and their corresponding time sequences.

\rev{
\section{Algorithm Pseudo Codes}\label{app:Algorithm Pseudo Codes}
In this appendix, we present Algorithms \ref{alg:FIFO} and \ref{alg:LNS} that provide pseudo codes for the FIFO and LNS algorithms, respectively, both of which are used in the \textit{Integrated} and \textit{AI} methods.

Furthermore, for the LNS, we specify two hyperparameters that we have used. The first is the number of LNS iterations; for our experiments, we used 70 iterations. However, it is recommended to maximize the number of iterations within the run time constraint. The second hyperparameter is the probability of accepting a new decision, as indicated in the second-to-last row of Algorithm \ref{alg:LNS}. In our experiments, based on trial and error, setting this probability to 0.7 provided the best results.
}

\begin{algorithm}[!t]
\color{safeblue}
\small
\DontPrintSemicolon
\caption{FIFO.}\label{alg:FIFO}
\KwIn{Current state $S_k$, New order $i_k$}
\KwOut{Updated schedule $({\Psi_k}^\prime, {t^s_k}^\prime, {\Theta_k}^\prime, {t^d_k}^\prime)$}
$\big({\Psi_k}^\prime, {t^s_k}^\prime, {\Theta_k}^\prime, {t^d_k}^\prime\big) \leftarrow \text{Planned schedule } 
\big({\Psi_k}, {t^s_k}, {\Theta_k}, {t^d_k}\big) \text{ from } S_k$\;
    $c \gets \text{first cook available } c \in C_{f_{i_k}}$\;
    $t \gets \max(t_k, \text{ availability time of } c$) \tcp*{set the starting preparation time}
    ${\textit{tardiness}}_{i_k} \gets \infty$ \; 
    $\text{Schedule } i_k \text{ at } t \text{ with cook } c$\;
    \For{$v \in V$}{
      \If{$v \text{ has no scheduled trips}$}{
            Continue\;}
     $r \gets \text{last scheduled trip of } v$\;
            \If{\text{number of orders in} $r < \kappa \text{ and its departure time} \geq t + t_i^p$}{
            $\text{Add } i_k \text{ to the trip, consider all feasible sequences within the trip and select the one with minimum total tardiness.}$
             \If{$\text{ a feasible trip was found and the tardiness for } i_k < \textit{tardiness}_{i_k}$}{              $\text{Update } \textit{tardiness}_{i_k}$\;
             $\theta^\prime \gets \textit{ trip } r \textit{ with } i_k \textit{ added to it such that the total tardiness is minimized}$
              }
              }
              }
    \If{$\textit{tardiness}_{i_k} = \infty$}
    {
        $v \gets \text{first vehicle that finishes its last scheduled trip}$\;
        $\text{departure time } \gets \max(t + t_{i_k}^p, \text{ the time }v \text{ finishes its last scheduled trip}$)\;
    $\theta^\prime \gets \text{a new trip consisting of only } i_k$\;
    \If{$\text{departure time} + t^t_{0{i_k}} > t + t_{i_k}^p +\delta_{f_{i_k}} \text{(i.e., the arrival time to }i_k \text{ violates freshness)}$}{
        $t \leftarrow \textit{departure time} + t^t_{0{i_k}} - \delta_{f_{i_k}}\text{ (i.e., delay preparation time to match freshness constraint)}$\;
        }
        }
$\text{Update } ({\Psi_k}^\prime, {t^s_k}^\prime, {\Theta_k}^\prime, {t^d_k}^\prime) \text{ according to } c, t, \theta^\prime,\text{ and }\textit{departure time}$\;
$\text{Return } ({\Psi_k}^\prime, {t^s_k}^\prime, {\Theta_k}^\prime, {t^d_k}^\prime)$\;
\end{algorithm}

\begin{algorithm}[!t]
\color{safeblue}
\small
\DontPrintSemicolon
\caption{The LNS.}\label{alg:LNS}
\KwIn{Current state $S_k$, New order $i_k$}
\KwOut{Updated schedule $({\Psi_k}^\prime, {t^s_k}^\prime, {\Theta_k}^\prime, {t^d_k}^\prime)$}
$(\Psi_k, t^s_k, \Theta_k, t^d_k) \leftarrow \text{Retrieve planned schedule } (\Psi_k, t^s_k, \Theta_k, t^d_k) \text{ from state } S_k$\;
$(\Psi_k, t^s_k, \Theta_k, t^d_k) \leftarrow \text{Update using FIFO with new order } i_k$  \tcp*{Obtain initial decision}
$({\Psi_k}^\prime, {t^s_k}^\prime, {\Theta_k}^\prime, {t^d_k}^\prime) \leftarrow (\Psi_k, t^s_k, \Theta_k, t^d_k)$ \tcp*{Set Best decision}
$(\widehat{\Psi}^x_k, \widehat{\Theta}^x_k) \leftarrow (\Psi_k, t^s_k, \Theta_k, t^d_k)$ \tcp*{Set Current partial decision}
\For{$\text{iter} \in \text{Number of LNS Iterations}$}{
    ${(\widehat{\Psi}^x_k, \widehat{\Theta}^x_k)}^\prime \gets \text{Apply a random operator to modify } \widehat{\Psi}^x_k \text{ or } \widehat{\Theta}^x_k$\tcp*{Set New partial decision}
    $\text{Check feasibility through PDFT and if feasible obtain full decision } ({\Psi_k}, {t^s_k}, {\Theta_k}, {t^d_k})^\prime$\;
    \lIf{$(\widehat{\Psi}^x_k, \widehat{\Theta}^x_k)^\prime \text{ is infeasible}$}{
        $\text{Continue to next iteration}$
    }
    $\text{Evaluate } \textit{New partial decision} \text{ using immediate cost + predicted cost to go (NN-based estimate)}$\;
    \If{$\text{Evaluated cost of } \textit{New partial decision} < \text{ cost of Current partial decision}$}{
        $(\widehat{\Psi}^x_k, \widehat{\Theta}^x_k) \leftarrow{(\widehat{\Psi}^x_k, \widehat{\Theta}^x_k)}^\prime$ \tcp*{Update Current partial decision}
        \If{$\text{New partial decision is the best so far}$}{
            $({\Psi_k}^\prime, {t^s_k}^\prime, {\Theta_k}^\prime, {t^d_k}^\prime) \gets ({\Psi_k}, {t^s_k}, {\Theta_k}, {t^d_k})^\prime$ \tcp*{Update Best decision}
        }
    }
    \Else{
         $(\widehat{\Psi}^x_k, \widehat{\Theta}^x_k) \leftarrow (\widehat{\Psi}^x_k, \widehat{\Theta}^x_k)^\prime \text{ with probability } 0.7$ \tcp*{Update Current partial decision}
    }
}   \vspace{0.1cm}
\text{Return Best decision } $({\Psi_k}^\prime, {t^s_k}^\prime, {\Theta_k}^\prime, {t^d_k}^\prime)$

\end{algorithm}

\color{black}
\vspace{2mm}
\section{Description of the Operators}\label{App: operators}

In the following, we define the seven operators used in our LNS procedure.
The first operator is motivated by an intuitive rule of thumb that it is better to prepare first orders with close deadlines and short preparation times. 
To that end, we define ${dl}_i$ as the (soft) deadline of the vehicle departure time with respect to order $i$ if it is the first stop of its assigned trip, i.e., ${dl}_i=t^o_{i}+\tau-{t}_{0i}^t$. Thus, if a vehicle departs with order $i$ later than ${dl}_i$, the order tardiness is positive. 
In addition, as we show in Claim \ref{th:Sequencing two similar orders}, this rule (described below) is the optimal policy for some cases of the problem.
\begin{claim}
 [Sequencing a set of similar orders]
\label{th:Sequencing two similar orders}
 At a given state $S_k$, consider  a set of orders of a certain food type $I_f$ that have not started preparation and are sorted such that ${dl}_i \leq {dl}_j $ and $t^p_i \leq t^p_{j}, \forall i < j \in I_f$. 
If, in sequencing the departures of each two orders $i<j$ in this set, the potential impact of the sequence on the departure time of orders that are not included in $I_f$ can be ignored, then, there exists an optimal decision in which each order $i<j$ is prepared no later than order $j$.  
 \end{claim}

\begin{proof}{Proof of Claim \ref{th:Sequencing two similar orders}.} 
First, note that for each two orders $i<j \in I_f$, if order $i$ arrives at the customer earlier than $j$, it can only improve the sum of tardiness over the orders in $I_f$, because ${dl}_i \leq {dl}_j$.
Since other orders that are not included in $I_f$ are not impacted according to the claim assumption, this is the case with respect to all orders as well. 
Second, given that order $i$ should be dispatched first, we can determine that its preparation should start no later than $j$ without worsening the decision because $t^p_i \leq t^p_j$.
\end{proof}

\emph{Operator 1 - Swap by deadlines and preparation times.} 
 \emph{Operator 1} selects a random food type $f \in F$, and then a not yet prepared order $i_{fj}$ to be swapped with the (not yet prepared) order that precedes it in $\psi_k^{f,x}$, $i_{f,{j-1}}$.
 Order $i_{fj}$ is selected according to weighted probabilities by the normalized values of ${dl}_i \cdot t^p_i, \forall i \in I_k^f$.
 Namely, the likelihood that order $i \in I_k^f$ will be chosen is $1-\frac{{dl}_i \cdot t^p_i}{\sum_{{i^\prime} \in I_k^f}{dl}_{i^\prime} \cdot t^p_{i^\prime}}$.
We obtain a \emph{new partial decision} in which ${\psi_k^{f,x}}^\prime = (i_{f1},\dots, i_{fj},i_{f,j-1},\dots)$ and the rest of the sequences remain as in the current partial decision. 

\emph{Operator 2 - Swap orders.} 
We select a random food type $f \in F$, and swap randomly (according to a uniform distribution) two of its (not yet prepared) orders $i_{fj}, i_{f{u}}$ in $\psi_{k}^{f,x}=(i_{f1}, \dots, i_{fj},\dots,i_{f{u}},\dots)$
 to obtain ${\psi_{k}^{f,x}}^\prime=(i_{f1},\dots, i_{fu},\dots,i_{fj},\dots)$. 

The third operator is motivated from both an intuitive rule of thumb of sequencing trips by SPT (shortest to longest travel times), and specific cases in which this rule is an optimal policy, see Claim \ref{th: Deployment of last orders} below. 

\begin{claim}
[Order departures by SPT]
\label{th: Deployment of last orders} 
Suppose that at a certain time when considering the departure times of a set of ready orders $O$, future orders potential arrival can be ignored (e.g., at time $T^c$).
Let $O^l \subset O$ be the set of orders which cannot be early in any feasible trip sequences decision, given all orders starting preparation times. 
Assume also that all the orders in $O^l$ are assigned to a trip without other orders.
 Then, there exists an optimal decision in which the orders in $O^l$ are sequenced by SPT (between them), where the entire sequence can also include orders in $O \backslash O^l$ in between them, if such a sequence is feasible. 
\end{claim}

\begin{proof}{Proof of Claim \ref{th: Deployment of last orders}.}
 Assume, by contradiction, that an optimal sequence of departures consists of two orders $ i,j \in O^l$, such that the sequence between them is not SPT although switching their places in the sequence is feasible. 
 Assume without loss of generality that $t_{0i}^t+t_{i0}^t<t_{0j}^t+t_{j0}^t$ and that $j$ is scheduled to depart before $i$. Thus, the departure sequence includes a set of orders $O_1$, followed by order $j$, then another set of orders $O_2$, followed by order $i$, and finally, there is another set of orders $O_3$. 
 Note that $O = O_1 \cup O_2\cup O_3 \cup \{i,j\} $ and that $O_1,O_2,O_3,\{i\}$ and $\{j\}$ are all disjoint.
 
 Then, replacing $j$ and $i$ in the decision will:
(1) Not affect the tardiness of the orders in $O_1$ and $O_3$;
(2) Decrease the times in which the orders in $O_2$ arrive to the customers, and hence, will not increase the sum of tardiness over these orders;
(3) Decrease the tardiness of $i$ by $t_{0i}^t+t_{i0}^t$ plus the sum of travel times of all orders in $O_2$.
and increase the tardiness of $j$ by $t_{0i}^t+t_{i0}^t$ plus the sum of travel times of all orders in $O_2$. 
Because $t_{0i}^t+t_{i0}^t<t_{0j}^t+t_{j0}^t$, the increase of tardiness of $j$ is smaller than the decrease of the tardiness of $i$. 
Therefore, the total sum of tardiness over all the orders is improved, in contradiction to the assumption that it is an optimal decision.
\end{proof}

\emph{Operator 3 - Sequence trips by SPT.}     
We select random (according to a uniform distribution) three consecutive trips $\theta_l, \theta_{l+1},\theta_{l+2} \in \widehat{\Theta}_k^x$. 
Then, in $\widehat{\Theta}_k^{x\prime}$ we sequence them by the shortest to the longest travel times along the trips.

     \emph{Operator 4 - Swap trips.} 
\rev{We randomly select two consecutive trips, $\theta_l$ and $\theta_{l+1}$ in $\widehat{\Theta}_k^x=(\theta_1, \dots,\theta_l, \theta_{l+1}, \dots)$, and swap them. This results in a new configuration, $\widehat{\Theta}_k^{x\prime}$, where the positions of $\theta_l$ and $\theta_{l+1}$ are interchanged, maintaining the order of all other trips, i.e., $\widehat{\Theta}_k^{x\prime}=(\theta_1, \dots,\theta_{l+1}, \theta_l, \dots) $.}
    
     \emph{Operator 5 - Bundle trips.} We bundle two consecutive trips $\theta_l, \theta_{l+1} \in \widehat{\Theta}_k^x$ that include together up to $\kappa$ orders. The trips are chosen randomly and a new trip $\theta^\prime$ that includes all the orders in $\theta_l$ and $\theta_{l+1}$ is generated. In particular, the orders in $\theta^\prime$ are sequenced such that the orders in $\theta_l$ are visited first, according to their sequence, and then the orders in $\theta_{l+1}$ are visited according to their sequence. 
    
     \emph{Operator 6 - Break a trip.} We break a random trip $\theta \in \widehat{\Theta}_k^x$ to two trips $\theta^\prime$ and $\theta^{\prime\prime}$ such that $\theta^\prime$ includes only the first order of $\theta$, and $\theta^{\prime\prime}$ includes the rest of the orders in $\theta$. 
     Then, $\widehat{\Theta}_k^{x\prime}=( \dots,\theta^\prime, \theta^{\prime\prime}, \dots)$.
    
     \emph{Operator 7 - Shuffle a trip:} \rev{We randomly select a trip from $\theta \in \widehat{\Theta}_k^x$, where each trip is chosen with equal probability. The sequence of orders within this trip is shuffled, i.e., the sequence in which orders are to be delivered is chosen randomly.}

\vspace{5mm}
\section{The Assignment and Timing Sub-Problem}\label{app: PDFT}

In this appendix we provide details regarding PDFT, which given a state $S_k$ in the RMD-GK and a partial decision ${\psi_{k}^{f,x}}, \forall f \in F$ and ${\widehat{\Theta}_k^x}$, determines whether the partial decision is feasible or not.
\rev{If it is feasible}, starting times of order preparations and departure times \rev{are determined}.
As mentioned, the PDFT solves the ATP, defined in \rev{Section} \ref{sec: PDFT}, to optimality.
To describe the PDFT we first provide preliminaries, including a DP formulation of the ATP.
We then provide theoretical supports to its optimality, and show that it is polynomial in the number of open orders.


We next define a deterministic dynamic program.
Without loss of generality, we assume that we first assign orders to cooks, and for each assigned order determine its starting time, and then assign trips to vehicles, and determine each trip's departure time.
That is, there are $|I_k|+|{\widehat{\Theta}_k^x}|$ decision points, indexed by $n$. 
The first $|I_k|$ decision points can be partitioned by the food types, such that we assign the orders ${\psi_{k}^{f,x}}$ of a food type $f \in F$, one after another according to their sequence, and then continue to the another food type until all orders in $I_k$ are assigned to cooks.
For convenience, we assume that the food types are ordered so that there is a strict link between the index of the decision point and the task that is assigned. 
The next $|{\widehat{\Theta}_k^x}|$ decision points related to the trips in ${\widehat{\Theta}_k^x}$ assignment, one after another according to their sequence.  
Consequently, we further denote by $i_n\in I_k$ the order that is associated with decision point $n,\forall n \leq |I_k|$ and 
by $\theta_n\in{\widehat{\Theta}_k^x}$ the trip that is associated with decision point $n,\forall n > |I_k|$.
In the notation provided in this formulation we use the convention that superscripts are symbols and subscripts are indexes.

\noindent \textbf{A Dynamic Program Formulation of The ATP}

\emph{States.} A state 
$s_n=(\{t^{CE}_{n,c}|c \in C\},\{t^{VE}_{n,v}|v\in V\},\{({lb}_{n,\theta}, {ub}_{n,\theta})|\theta\in\widehat{\Theta}_k^x\})$ includes three components. 
The first two are the times $t^{CE}_{n,c}$ and $t^{VE}_{n,v}$ each cook $c \in C$ and each vehicle $v\in V$ is eligible to start the next task according to its availability (at the completion of the previously assigned tasks) and also subject to the sequence constraints (no earlier than the starting time of the last assigned order (of the respective food type) or trip according to the given sequences).
Note that the symbols "CE" and "VE" stand for Cook and Vehicle Eligibility, respectively. 
The third component is a lower and upper bound $({lb}_{n,\theta}, {ub}_{n,\theta})$  on the departure time of each trip in ${\widehat{\Theta}_k^x}$ that was not scheduled yet. 
 For each trip $\theta$, ${lb}_{n,\theta}$ represents the earliest time all the orders that were already assigned to the cooks and are included in $\theta$, have finished preparation and ${ub}_{n,\theta}$ represents the latest time all these orders can depart without violating any freshness constraint. 
 The values of $({lb}_{n,\theta},{ub}_{n,\theta}), \forall \theta \in {\widehat{\Theta}_k^x}$ also reflect the trips sequence as detailed in the sequel.

\emph{Decisions.} 
We next describe the possible decisions at the first $|I_k|$ decision points and at the remaining $|{\widehat{\Theta}_k^x}|$ decision points, respectively.
A decision $a_n=(a^C_n,a^t_n), \forall 1\leq n \leq |I_k|$ includes an assignment of order $i_n$ 
to a cook $a^C_n\in C_{f_{i_n}}$, and a starting time $a^t_n$.
The starting time $a^t_n$ must satisfy Feasibility Conditions 1-4 with respect to the orders assignments and preparation starting times (each is a necessary but not a sufficient condition), listed below.
\begin{enumerate}
    \item[\textbf{Feasibility Condition 1.}] 
    The first feasibility condition is related to the time the cooks become available. That is, for each assignment $a^C_n\in C_{f_{i_n}}$ 
     the preparation starting time cannot be earlier than the time cook $a^C_n$ is available, i.e., 
    \begin{equation} \label{eq: feasible order time 1}
    a^t_n\geq t^{CE}_{n,a^C_n}
    \end{equation}
    \end{enumerate}
 
    We denote the index of the trip (within the sequence of trips) that contains order $i_n$ by $q(i_n)$, therefore $\theta_{q(i_n)}$ is the trip that contains order $i_n$. 
    Note that this is in general different than the trip on which we decide in decision point $n$, which was denoted by $\theta_n$.
\begin{enumerate}
    \item[\textbf{Feasibility Condition 2.}] 
    The second condition considers the influence of the preparation starting time on the departure time window of the order's trip.
    Namely, the starting time of order $i_n$ must keep a feasible departure time window to its trip, $\theta_{q(i_n)}$, by ensuring that order $i_n$ can be dispatched in at least one point in time within [${lb}_{n,\theta_{q(i_n)}}$,${ub}_{n,\theta_{q(i_n)}}$].
That is, $i_n$ cannot start preparation too early so it can arrive fresh if departs not earlier than ${lb}_{n,\theta_{q(i_n)}}$, and it must be prepared no later than ${ub}_{n,\theta_{q(i_n)}}$, as expressed in
Eq.(\ref{eq: feasible order time 2}) below (recall that $t_{i_n}(\theta_{q(i_n)})$ denotes the travel time to order $i_n$ along trip $\theta_{q(i_n)}$).
\begin{equation} \label{eq: feasible order time 2}
{lb}_{n,\theta_{q(i_n)}} + t_{i_n}(\theta_{q(i_n)})-p^t_{i_n} - \delta_{f_{i_n}} \leq 
a^t_n\leq 
{ub}_{n,\theta_{q(i_n)}} - p^t_{i_n}    
\end{equation}

\item[\textbf{Feasibility Condition 3.}] 
The third feasibility condition is related to the possible influence of the starting time of order $i_n$ on the departure times of the trips that appear in ${\widehat{\Theta}_k^x}$ later than its trip $\theta_{q(i_n)}$.
Namely, since these trips cannot depart earlier than $\theta_{q(i_n)}$, a necessary condition for a feasible departure time window for each $\theta^\prime$ that appears later than $\theta_{q(i_n)}$ in ${\widehat{\Theta}_k^x}$, is:
\begin{equation} \label{eq: feasible order time 3}
a^t_n \leq {ub}_{n,\theta^\prime}-p^t_{i_n} 
\end{equation}

\item[\textbf{Feasibility Condition 4.}]
The fourth feasibility condition is related to the limited number of vehicles as well as the given departure sequence.
It makes use of the following claim:
\begin{claim}\label{th: |V| subsequent trips overlapping}
Let $\theta_{m}$ be some trip in $\widehat{\Theta}_k^x$ and let $t^t(\theta_m)$ denote the vehicle driving time along trip $\theta_m$ from the depot and back. 
Then, the departure time of $\theta_{m}$ cannot be earlier than 
$\min_{l \in \{m-|V|-1,\dots,m-1\}}({lb}_{\theta_{l}}+t^t(\theta_{l}))$.
\end{claim}
\begin{proof}{Proof of Claim \ref{th: |V| subsequent trips overlapping}.}
First, consider the case in which each of the $|V|$ trips $\{\theta_{l}| l \in \{m-|V|-1,\dots,m-1\}$ is assigned to a different vehicle $v\in V$, then the first vehicle to return is available again not earlier than $\min_{l \in \{m-|V|-1,\dots,m-1\}}({lb}_{\theta_{l}}+t^t(\theta_{l}))$.
Now, consider the case in which at least two of these trips are assigned to the same vehicle, then, the departure time of the second trip assigned to the same vehicle with another trip is not earlier than $\min_{l \in \{m-|V|-1,\dots,m-1\}}({lb}_{\theta_{l}}+t^t(\theta_{l}))$.
Thus, because of the sequence constraints, the departure time of $\theta_{m}$ also cannot be earlier than 
$\min_{l \in \{m-|V|-1,\dots,m-1\}}({lb}_{\theta_{l}}+t^t(\theta_{l}))$. 
\end{proof}

Consequently, $a_n^t$ should be determined such that the following feasibility condition is satisfied $\forall \theta_{m} \in \{\theta_{{q(i_n)}+1},\dots,\theta_{{q(i_n)}+|V|+1}\}$:
\begin{equation}
\min(a_n^t+p_{i_n}^t+t^t(\theta_{q(i_n)}), \min_{l \in \{m-|V|-1,\dots,m-1\}}({lb}_{\theta_{l}}+t^t(\theta_{l})))
\leq
{ub}_{n,\theta_{m}} 
\end{equation}
\end{enumerate}

We note that Feasibility Condition 1 depends on the exact cook assignment, whereas Feasibility Conditions 2-4 are independent of it. 
This leads to the following claim.

\begin{claim}\label{th: indifferent regarding the specific cook}
Given a value of $a^t_n$ that satisfies Feasibility Conditions 2-4, we are indifferent regarding the specific cook assignment as long as Feasibility Condition 1 is satisfied.
\end{claim}
\begin{proof}{Proof of Claim \ref{th: indifferent regarding the specific cook}.}
Given a value of $a^t_n$ that satisfies Feasibility Conditions 2-4, let $C^A=\{c\in C_{f_{i_n}}| t^{CE}_{n,c}\geq a^t_n \}$ be the set of cooks that order $i_n$ can be assigned to at time $a^t_n$. 
Without loss of generality, assume that $i_n$ is assigned to cook $c\in C^A$, which will then be available for the following orders at time $a^t_n + p^t_{i_n}$, while the other cooks in $C^A$ will be ready for the following orders at time $a^t_n$ (to preserve the preparation sequence).
Since these cooks are homogeneous, the resulting states for any other cook assignment have the same value because the states are symmetric (i.e., the set of cooks eligibility times $\{t^{CE}_{{n+1},c}| c\in C^A\}$ are identical but the cook indices of these times may differ).
\end{proof}

Consequently, if there is a value of $a^t_n$ that satisfies Feasibility Conditions 2-4, we can select arbitrarily one of the cooks that are available at $a^t_n$ (Feasibility Condition 1).
In this case, in order to reduce both the decision space and the state space, we always assign the order to the cook with the lowest index.

If there is no decision that satisfies all Feasibility Conditions 1-4, then, there is no feasible decision at state $s_n$. 
We next focus on decisions concerning trip assignments, i.e., when $n > |I_k|$.

A decision $a_n=(a^V_n,a^t_n), \forall n > |I_k|$ includes an assignment of trip $\theta_n$ to a vehicle $a^V_n \in V$, and its departure time, $a^t_n$. 
The departure time $a^t_n$ must satisfy Feasibility Conditions 5-7, listed below (each is a necessary but not a sufficient condition):
\begin{enumerate}
\item[\textbf{Feasibility Condition 5.}] For each assignment $a^V_n\in V$, the trip can only be assigned after vehicle $a^V_n$ becomes available, i.e.,
\begin{equation} \label{eq: feasible order time 5}
 a^t_n \geq t^{VE}_{n,a^V_n}  
\end{equation}

\item[\textbf{Feasibility Condition 6.}] 
A trip can only depart within its departure time window:
\begin{equation} \label{eq: feasible order time 6}
{lb}_{n,{\theta}_n} \leq a^t_n\leq {ub}_{n,{\theta}_n}   
\end{equation}

\item[\textbf{Feasibility Condition 7.}] 
This feasibility condition is related to the effect of the departure time of trip $\theta_n$ on the trips that follow it, which cannot depart earlier than $\theta_n$.
Thus, the upper bound on each of these trips' departure time window must be larger than $a^t_n$.
Hence, 
\begin{equation} \label{eq: feasible order time 7}
a^t_n \leq 
{ub}_{n,\theta_l}, 
\textit{ }
 \forall \theta_l \in \widehat{\Theta}_k^x, l>n
\end{equation}

\end{enumerate}

Similarly to Feasibility Conditions 1-4, we note that Feasibility Condition 5 depends on the exact vehicle assignment while Feasibility Conditions 6-7 are independent of it. 
This leads to the following claim. 
\begin{claim}\label{th: indifferent vehicles}
Given a value of $a^t_n$ which satisfies that ${lb}_{n,{\theta}_n} \leq a^t_n \leq \min_{\theta_l \in \{\widehat{\Theta}_k^x| l\geq n\}}{{ub}_{n,\theta_l}}$, i.e., Feasibility Conditions 6-7, we are indifferent regarding the specific vehicle assignment, as long as Feasibility Condition 5 is satisfied.
\end{claim}
\begin{proof}{Proof of Claim \ref{th: indifferent vehicles}.}
The proof is based on similar statements as in the proof of Claim \ref{th: indifferent regarding the specific cook} but with respect to the vehicles instead of the cooks.
\end{proof}

Then, if there is a value of $a^t_n$ that satisfies both Feasibility Conditions 6 and 7, we assign trip $\theta_n$ to the vehicle with the lowest index that is available at $a^t_n$ (Feasibility Condition 5).
This observation further reduces the decision space and due to state symmetry also the state space.
However, if there is no decision that satisfies all Feasibility Conditions 5-7, then, there is no feasible decision at state $s_n$. 

\emph{Transitions.} 
We denote by $H(s_n,a_n)$ the transition from a state $s_n$ and a decision $a_n$ to state $s_{n+1}$.
We start by describing the transition for $n \leq |I_k|$.
The availability time of cook $a^C_n$ is updated to
$t^{CE}_{{n+1},a^C_n}=a^t_n+p^t_{i_n}$, 
where the availability of the rest of the cooks of food type $f_{i_n}$ is updated to
$t^{CE}_{{n+1},c} = \max(t^{CE}_{n,c}, a^t_n), \forall c\neq a^C_n \in C_{f_{i_n}}$ because the next orders of food type $f_{i_n}$ cannot start earlier than $a^t_n$.
We update the feasible departure time window of the trip $\theta_{q(i_n)}$ that includes $i_n$, by setting
${lb}_{n+1,\theta_{q(i_n)}}=\max({lb}_{n,\theta_{q(i_n)}},a^t_n +p^t_{i_n})$ 
and 
${ub}_{n+1,\theta_{q(i_n)}}=\min({ub}_{n,\theta_{q(i_n)}}, \delta_f-t_{i_n}(\theta_{q(i_n)})+a^t_n + p^t_{i_n})$ 
(the right expression is the latest time order $i_n$ can depart and still satisfy the freshness constraint).
We also update the trip time window of each trip $\theta_m, \forall m < q(i_n)$ that appears earlier than $\theta_{q(i_n)}$ in the sequence so that ${ub}_{{n+1},{\theta_m}} = \min({ub}_{{n,{\theta_m}}},{ub}_{n+1,\theta_{q(i_n)}})$ because of sequence constraints.
Similarly, we update the trip time window of each trip $\theta_m, \forall m > q(i_n)$ that appears later than $\theta_{q(i_n)}$ in the sequence so that ${lb}_{{n+1},{\theta_m}} = \max({lb}_{{n,{\theta_m}}},{lb}_{n+1,\theta_{q(i_n)}})$. 
We may also update further the trip time windows of some of the trips that appear later than trip $\theta_{q(i_n)}$ in the sequence $\widehat{\Theta}_k^x$, according to Claim \ref{th: |V| subsequent trips overlapping}.
That is, following each update of ${lb}_{n+1,\theta_m}$ of trip $\theta_m$, the lower bound of the departure time window of each trip $\theta_{m^\prime} \in \{\theta_{m+1},\dots,\theta_{m+|V|+1}\}$ must also be updated to $\min_{l \in \{m^\prime-|V|-1,\dots,m^\prime-1\}}({lb}_{n+1,\theta_{l}}+t^t(\theta_{l}))$.

We next describe the transition for $n > |I_k|$. 
The vehicles availability times are updated to
$t^{VE}_{{n+1},a^V_n}=a^t_n+t^t(\theta_n)$
and $t^{VE}_{{n+1},v}=\max(t^{VE}_{n,v},a^t_n), \forall v\neq a^V_n \in V$ (the last transition guarantees that trips are dispatched according to their sequence).

\emph{Costs.} The immediate cost $c(s_n,a_n)$ for state $s_n$ and decision $a_n$, is zero for $n \leq |I_k|$, and the total delay of the orders included in trip $\theta_n$ for $|I_k| < n\leq |I_k|+|\widehat{\Theta}_k^x|$, with an exception that if there is no feasible decision at state $s_n, \forall n \leq |I_k|+|\widehat{\Theta}_k^x|$ the immediate cost is $\infty$.


\emph{The Bellman Equation.} 
Denote the optimal value of state $s_n$ by $J(s_n)$, then, the Bellman Equation is given by:
\begin{equation}\label{eq:bellman ATP}
J(s_n)=\min_{a_n}\{c(s_n,a_n)+J(H(s_n,a_n))\}, \textit{ } \forall s_n    
\end{equation}

\emph{Termination.} We terminate the program at decision point $n+1$ if there is no feasible decision at state $s_n$ and/or if $n = |I_k|+|\widehat{\Theta}_k^x|$ (meaning, all tasks are assigned).

\emph{The Initial State.} We next define the initial state variables of the ATP given a state $S_k$ of the RMD-GK problem and a partial decision of it, ${\psi_{k}^{f,x}}, \forall f \in F$ and ${\widehat{\Theta}_k^x}$.
Referring to the initial state as $n=1$, the time at which each cook $c\in C$ becomes available, $t^{CE}_{1,c}\in[t_k, T]$, is the maximum between the finish time of the last order that started preparation before $t_k$ by cook $c$ and time $t_k$.
Similarly, the time at which each vehicle $v\in V$ becomes available, $t^{VE}_{1,v}$, is the maximum between the return time from its last trip that departed before $t_k$ and time $t_k$ (given by $t^r_k$ in state $S_k$).

To define the initial departure time windows, let $I^0_k \subseteq I_k$ denote the set of open orders that at time $t_k$ already started or finished preparation (but were not yet dispatched).
Recall that these orders' starting preparation times are given by $t_{ik}^{s}, \forall i\in I^0_k$.
Thus, for trips in $\widehat{\Theta}_k^x$ associated with orders in $I^0_k$ we set the initial time windows accordingly.
That is, let $I^0_k(\theta)\subseteq I^0_k$ be those orders that are included in trip $\theta$.
Then, for each $\theta \in \widehat{\Theta}_k^x$,
${lb}_{1,\theta}=
\max(\min_{v\in V}(t^{VE}_{1,v}),\max_{i\in I^0_k(\theta)}{t_{ik}^{s}+t^p_i})$
and
${ub}_{1,\theta}=\min(T, \min_{i\in I^0_k(\theta)}{t_{ik}^{s}+t^p_i+\delta_{f_i}-t_i(\theta)})$.

\vspace{2mm}

\noindent\textbf{The PDFT-algorithm}

\noindent The PDFT algorithm is based on the above dynamic program. 
We next list the PDFT steps to achieve the optimal solution to the ATP:

\begin{itemize}
\item[1.] 
Set $n=1$ and obtain $s_1$ from $S_k$ and the given partial decision in the RMD-GK. 
Verify that all trips have a lower bound that is lower than or equal to their corresponding upper bound, otherwise no feasible decision in the ATP exists.
\item [2.] For state $s_n$, if there exists a feasible decision $a_n$ that satisfies all Feasibility Conditions 1-4, set the starting time, $a_n^t$, as early as possible and
skip to [4].
\item [3.] Else (if no feasible decision exists at $s_n$): 
It means that the earliest time $i_n$ can be ready for departure according to Feasibility Condition 1 is too late with respect to the upper bounds on the departure times of its trip and/or the trips that follow it (${ub}_{{n},\theta_{q({i_n})}}, {ub}_{{n},\theta_{q({i_n})+1}},\dots$) according to Feasibility Conditions 2-4. Thus, let decision point $n^\prime(\leq n)$ be the earliest decision point in which the assignment of an order that belongs to one of the trips $\theta_{q(i_n)},\theta_{q(i_n)+1},\dots$ is considered. Then:
\begin{itemize}
\item [3.1.] 
If $n^\prime=n$: No feasible decision exists to the ATP because it means that no earlier decisions regarding orders that belong to these trips were made. That is, the upper bounds are only affected by orders that started preparation before $t_k$, which cannot be postponed. 
Then, the PDFT is terminated with a conclusion that the given partial decision is infeasible.
\item [3.2.] If $n^\prime<n$: Return to [2] with $n={n^\prime}$, and state $s_{n^\prime}$ with the values that prevailed when it was last considered 
 and set 
 ${lb}_{{n^\prime},\theta_{q({i_n})}}=\min\{t_{n,{c}}^{CE}| c \in C_{f_{i_n}}\} + p_{i_n}^t$ (the earliest time that $i_n$ can be ready for departure given the cooks availability), and for each trip that appears later than trip $\theta_{q(i_n)}$ in the sequence, $\theta_m\textit{ } (m>q(i_n))$ set ${lb}_{{n^\prime},\theta_m} = \max({lb}_{{n^\prime},\theta_{q(i_n)}},{lb}_{{n},\theta_m})$.
\end{itemize}
\item[4.] Obtain state $s_{n+1} = H(s_n,a_n)$ and set $n=n+1$. 
If $n \leq|I_k|$ return to [2] (otherwise go to [5]).
\item[5.] For state $s_n$, if there is a feasible decision $a_n$ that satisfies all Feasibility Conditions 5-7, set $a_n^t$ as early as possible.
Skip to [7].
\item [6.] Else (if no feasible decision exists), it means that the earliest time a vehicle is available is too late with respect to ${ub}_{{n},\theta_n}$.  
Let decision point $n^\prime$ be the earliest decision point in which the assignment of an order that belongs to one of the trips $\theta_n,\theta_{n+1},\dots$ is considered.
Return to [2] with $n={n^\prime}$, and state $s_{n^\prime}$ with the values that prevailed when it was last considered but set
 ${lb}_{{n^\prime},\theta_{n}}=\min\{t_{n,{v}}^{VE}| v \in V\}$ (the earliest time a vehicle is available for $\theta_n$) and 
for each subsequent trip $\theta_m \textit{},m>n$ set ${lb}_{{n^\prime},\theta_m} = \max({lb}_{{n^\prime},\theta_n},{lb}_{{n},\theta_m})$.
\item[7.] Obtain state $s_{n+1} = H(s_n,a_n)$ and set $n=n+1$. 
If $n \leq|I_k|+|\widehat{\Theta}_k^x|$ return to [5]. 
Otherwise, terminate and return the assignments and starting times.
\end{itemize}

Due to the sequence constraints inherent in the ATP, the PDFT is designed to prioritize starting tasks as early as possible, provided the freshness constraint is satisfied. 
The next definition and claims show that the PDFT obtains the optimal solution to the ATP.

\begin{definition}[Feasible state]\label{def: feasible state}
We say that a state $s_n$ is feasible if there exists a feasible policy (a sequence of feasible decisions) from state $s_n$ until the last state $s_{|I_k|+|\widehat{\Theta}_k^x|}$, i.e., until all orders and trips are assigned.
\end{definition}
Note that if the freshness constraint is not effective, it is optimal to start all tasks in the ATP as soon as possible. 
The only reason to delay starting times is synchronization with other orders and vehicles so that the freshness constraints are satisfied.
Hence,
\begin{claim}\label{th: greedy best earliest preparation times}
 Let $A$ be the set of all possible preparation starting times of $i_n$ in feasible policies from $s_n$ to $s_{|I_k|+|\widehat{\Theta}_k^x|}$, for a feasible state $s_n, n \leq |I_k|$.
Then, there exists an optimal policy in which $a_n^t=\min(A)$.
\end{claim}
\begin{proof}{Proof of Claim \ref{th: greedy best earliest preparation times}.}
Given the feasibility of state $s_n$, among all possible policies it is desired to implement the earliest starting time because the sooner order $i_n$ is scheduled, the sooner trip $\theta_{q(i_n)}$ as well as its subsequent trips can depart, and the sooner the subsequent orders (and their associated trips) can start.
This, aligns with the objective of minimizing delays.
\end{proof}

\begin{claim}\label{th: greedy best earliest departures given preparation starting times}
For a state $s_n, n> |I_k|$, if there exists a feasible decision $a_n$, the optimal decision is to depart trip $\theta_n$ as early as possible, i.e., at time $a_n^t=\max(\min\{t^{VE}_{n,v}| \forall v\in V\},{lb}_{n,\theta_n})$.
\end{claim}
\begin{proof}{Proof of Claim \ref{th: greedy best earliest departures given preparation starting times}.}
First, according to Feasibility Conditions 5-7, $a_n^t\geq\max(\min\{t^{VE}_{n,v}| \forall v\in V\},{lb}_{n,\theta_n})$, therefore this decision is feasible.
Second, it is trivial that the sooner trip $\theta_n$ departs, the lower the total tardiness of the orders included in it.
Third, the earlier the trip departs, the subsequent states are less constrained (see the \emph{Transition} description above).
\end{proof}

\begin{claim}\label{th: optimal algorithm}
The PDFT provides an optimal solution to the ATP.
\end{claim}

\begin{proof}{Proof of Claim \ref{th: optimal algorithm}.} 
For each feasible state \(s_n\), an optimal policy in the set of all possible policies from $s_n$, is to set \(a^t_n\) as early as possible (Claims \ref{th: greedy best earliest preparation times} and \ref{th: greedy best earliest departures given preparation starting times}).
The PDFT systematically explores feasible policies from each state, prioritizing the earliest possible starting times at each decision point. Consequently, the only feasible policy identified by the algorithm ensures that starting times at each feasible state are as early as possible.
That is, at each state $s_n$ along the algorithm in steps 2 and 5, if possible, we assign tasks at the earliest possible starting time according to the state's feasibility conditions.
Otherwise, if no feasible decision exists at $s_n$, in steps [3] and [6] we return to $n^\prime$, the first decision point whose starting time postponement may enable a feasible starting time for order $i_n$. 
From that point on, we update the value of the state variables to account for the feasibility of order $i_n$ and continue from that point on, as previously. 
Note that this means that we refer to state $s_n$ in multiple iterations even though its values change from one iteration to the next.
We also note that for $n\leq |I_k|$ the earliest time order $i_n$ can be prepared by a cook, $\min\{t_{n,{c}}^{CE}| c \in C_{f_{i_n}}\}$, is a lower bound on the earliest time order $i_n$ can start preparation in any feasible policy because all orders $i_1,\dots,i_{n-1}$ assigned as soon as possible in all earlier decisions.
Similarly, for $n > |I_k|$ the earliest time trip $\theta_n$ can be departed by a vehicle, $\min\{t_{n,{v}}^{VE}| v \in V\}$, is a lower bound on the earliest time it can depart in any feasible policy because all orders and trips started as soon as possible in all earlier decisions.
Thus, the changes made in the values of the state variables of state $s_{n^\prime}$ are necessary in each feasible policy that starts in $s_{n^\prime}$.
\end{proof}

\begin{claim}\label{th: stopping condition}
  If a state $s_n$ has no feasible decision, and then this state is visited again after going back to state $s_{n^\prime}$ (with the described updates in Step [3] or [6]), 
  and again has no feasible decision, then, no feasible decision exists in the ATP. 
\end{claim}
\begin{proof}{Proof of Claim \ref{th: stopping condition}.}
Note that in the case described, some orders were postponed to allow a later departure of order $i_n$ (for $n \leq |I_k|$) or trip $\theta_n$ (for $n > |I_k|$).
Thus, if no feasible decision exists at the new $s_n$, it means that the eligibility times of cooks of the food type of order $i_n$ (for $n \leq |I_k|$) or the vehicles (for $n > |I_k|$), were also postponed so that the corresponding resources are now eligible later than in the former $s_n$.
This means that some of the postponed orders still cannot depart when order $i_n$ can be ready (for $n \leq |I_k|$) or when the first available vehicle can depart for their corresponding trip (for $n > |I_k|$).
Then, postponing again these orders will postpone again the corresponding eligibility times so that again no feasible decision will be obtained.
\end{proof}

\begin{claim}\label{th: the PDFT polynomial}
The PDFT is polynomial in the number of open orders, $|I_k|$.
\end{claim}
\begin{proof}{Proof of Claim \ref{th: the PDFT polynomial}.}
Each of the seven algorithm steps has a complexity of $O(|I_k|)$.
In the worst case at each decision point $n$ ($ 1\leq n\leq |I_k|+|\widehat{\Theta}_k^x|$) we return to decision point $1$ once (Claim \ref{th: stopping condition}), i.e., the number of algorithm iterations cannot exceed $1+2+\dots+(|I_k|+|\widehat{\Theta}_k^x|)=O(|I_k|+|\widehat{\Theta}_k^x|)^2=O(|I_k|^2)$ because the number of trips is not larger than the number of orders.
Thus, the complexity of the PDFT is $O(|I_k|^3)$.
\end{proof}

\begin{claim}\label{th: PDFT determines feasibility}
 The PDFT determines that a partial decision is feasible if and only if it is feasible.    
\end{claim}
\begin{proof}{Proof of Claim \ref{th: PDFT determines feasibility}.}
Note that the PDFT only determines that a partial decision is not feasible in the case described in Claim \ref{th: stopping condition}, otherwise it obtains a feasible (optimal) solution.
\end{proof}

\section{Implementation Details of the VFA}
In this section, we provide detail on our value function approximation (VFA).

\subsubsection*{VFA Features.}\label{app: features}
Our features are designed with two requirements in mind. 
First, they must allow for transfer learning, i.e., the dimension of the feature vector must be independent of the expected demand, the number of cooks employed, and the size of the delivery fleet. 
\rev{This means that the dimensionality of our feature vector remains constant at 21 dimensions. This fixed dimensionality ensures that our model can generalize across different scenarios without needing reconfiguration for each unique set of operational parameters, thereby enhancing the model's adaptability and scalability.}
Second, the features must represent currently available resources, corresponding to our cooks and our fleet, in each given post-decision state.
\rev{This includes the current status of our cooks and fleet, which are critical to optimizing operational decisions. To capture the dynamics of our resources effectively, we aggregate data into summary statistics, providing a comprehensive yet concise overview of each resource's state. These statistics include the mean, maximum, and minimum values, which offer insights into the distribution and extremities of workload and availability among cooks and vehicles. For instance, understanding the range of workloads among cooks (from least to most burdened) helps in balancing the allocation of new orders more effectively, preventing bottlenecks and improving overall efficiency.
Table~\ref{tab:feature_summary} below summarizes the specific features employed in our model.}

\begin{table}[!t]
\caption{Summary of the employed features.}\label{tab:feature_summary}
\small
\centering
\begin{tabular}{lll}
\cline{1-2}
Feature Description               & Summary Statistics     &  \\ \cline{1-2}
Current time                      &                        &  \\
Percentage of idle cooks          &                        &  \\
Number of orders per cook         & Mean, Maximum, Minimum &  \\
Total work time per cook          & Mean, Maximum, Minimum &  \\
Scheduled finishing time per cook & Mean, Maximum, Minimum &  \\
Percentage of idle vehicles       &                        &  \\
Time of final return 
per vehicle           & Mean, Maximum, Minimum &  \\
Number of trips per vehicle      & Mean, Maximum, Minimum &  \\
Number of orders per vehicle      & Mean, Maximum, Minimum & 
\end{tabular}
\end{table}

\subsubsection*{VFA Architecture.}\label{app: architecture}
We employ a fully-connected feed-forward neural network to approximate the value function $V$. The network consists of an input layer with 21 nodes, two hidden layers with 256 nodes each, and an output layer with one node. We use a ReLU activation function for input layer and hidden layers and no activation function for the output layer. The initial weights and biases of each node are initialized according to the Kaiming He initialization \citep{he2015delving}.

\subsubsection*{VFA Training Framework.}\label{app: training}
The VFA is trained during an extensive offline simulation. For this purpose, all tuples of post-decision states and cumulative delays originating from each post-decision state are saved to an experience replay \citep{lin1992self}. The size of the experience replay is set to one million tuples. After each day, we sample a batch of 128 experiences from the experience replay to train the VFA. We choose a constant learning rate of 0.001 and employ adaptive moment estimation \citep{kingma2015adam} to update all network parameters. Due to the high stochasticity of the restaurant meal delivery with ghost kitchens problem, we do not follow a common $\epsilon$-greedy strategy. Instead, we always search for decisions based of the current VFA parameterization.
\rev{
Algorithm \ref{alg:training} summarizes the NN training steps. 
}
\begin{algorithm}[!t]
\color{safeblue}
\small
\DontPrintSemicolon
\caption{NN training algorithm.}\label{alg:training}
\KwIn{Experience replay $\mathcal{E}$ with capacity $1,000,000$ tuples}
\KwOut{Optimized neural network parameters $\Theta$}
Initialize neural network with input layer (21 nodes), two hidden layers (256 nodes each), and output layer (1 node)\;
Initialize all weights and biases $\Theta$ using Kaiming He initialization\;
Set learning rate to $\alpha = 0.001$\;
\While{not converged}{
    
    $\mathcal{B} \gets \text{SampleBatch}(\mathcal{E}, 128)$ \tcp*{Sample a batch from the experience replay} 
    \For{$(\text{post decision state}, \text{cumulative delay}) \in \mathcal{B}$\tcp*{Perform one training step on the batch}} {
        $y \gets \text{cumulative delay}$\;
        $V_{\Theta}(\text{post decision state}) \gets \text{Forward pass through network}$\;
        $\mathcal{L}(\Theta) \gets \text{Compute loss}(\text{MSE}(V_{\Theta}(\text{post decision state}), y))$\;
        $\text{Backpropagate loss and update parameters using Adam optimizer}$\;
    }
    \If{convergence criteria met}{
        break\;
    }
}
\Return{$\Theta$}\;
\end{algorithm}

\section{Generation of Instances}\label{app:instance_generation}
We first, define how delivery requests are sampled, before we detail how preparation times are sampled.
The procedure for sampling all 
orders 
over a day is given by Algorithm~\ref{alg:customer_requests}. 
The procedure is summarized as follows: First, we sample the number of 
orders 
at lunch time $n_l\sim\mathcal{N}\left(\mu_{n,l}, \sigma_{n,l}\right)$ and dinner time $n_d\sim\mathcal{N}\left(\mu_{n,d}, \sigma_{n,d}\right)$. 
We set $\sigma_{n,l}=\frac{1}{40}\mu_{n,l}$ and $\sigma_{n,d}=\frac{1}{40}\mu_{n,d}$. 
Second, for each order 
$k\in\{1,\dots,n_l\}$, we sample an arrival time $t_k$ from the lunch time distribution and for each order 
$k\in\{n_l+1,\dots,n_d\}$, we sample an arrival time $t_k$ from the dinner time distribution. Third, we sample a customer location $V(c_k)$ for each order. 
The location $V(c_k)$ depends on the order's arrival 
time $t_k$.
The closer the order's arrival 
time is to the lunch time peak, the higher the chance that we resample the customer location with a fixed resampling probability $\rho$ if the customer location is not within the inner city nodes. 
Analogously, the probability of resampling with the same fixed resampling probability $\rho$ increases the closer the order's arrival 
time is to the dinner time peak. However, in this case we only resample if the initial customer location is within the inner city nodes.
\begin{algorithm}[!t]
\small
\DontPrintSemicolon
\caption{Sampling customer requests.}\label{alg:customer_requests}
\KwIn{customer locations $C\subset V$, inner city nodes $V^\prime\subset V$, demand parameters $(\mu, \sigma)$, resampling rate $\rho\in (0,1)$}
\KwOut{Set of customer requests $R\subset [0,T]\times V$}
$R\gets\emptyset$\;
$n_l\sim\mathcal{N}\left(\mu_{n,l}, \sigma_{n,l}\right)$ \tcp*{number of lunch-time requests}
$n_d\sim\mathcal{N}\left(\mu_{n,d}, \sigma_{n,d}\right)$ \tcp*{number of dinner-time requests}
$t_1,\dots,t_{n_l}\sim\mathcal{N}\left(\mu_{t,l}, \sigma_{t,l}\right)$\tcp*{time of lunch-time requests}
$t_{n_l+1},\dots,t_{n_d}\sim\mathcal{N}\left(\mu_{t,d}, \sigma_{t,d}\right)$\tcp*{time of dinner-time requests}
\For{$k\gets 1$ \KwTo $n_l+n_d$}{
$\epsilon\sim U[0,1]$\;
$V(c_k)\sim U(C)$\tcp*{sample customer location}
\eIf{$\epsilon < (1080 - t_k) / 360$}
{
    $\epsilon\sim U[0,1]$\;
    \If{$V(c_k)\not\in V^\prime$ and $\epsilon<\rho$}
    {
    $V(c_k)\sim U(C)$\tcp*{resample customer location}
    }
  }{
    $\epsilon\sim U[0,1]$\;
    \If{$V(c_k)\in V^\prime$ and $\epsilon<\rho$}
    {
    $V(c_k)\sim U(C)$\tcp*{resample customer location}
    }
}
$R \gets R\cup \{(t_k, V(c_k))\}$\tcp*{update customer requests}
}
\end{algorithm}
The final set of orders is realized according to their arrival 
times. When sampling customer locations in our case study, we set the resampling rate to $\rho=0.5$.

We assume that meal preparation times in minutes are distributed according to a log-normal distribution to match the long-tails in real-world preparation times. We choose $\mathcal{LN}\left(\mu(\text{pt})_j, \sigma(\text{pt})_j\right)$, $\forall j\in J$ with $\mu(\text{pt})_1=10, \sigma(\text{pt})_1=1.5$, $\mu(\text{pt})_2=9, \sigma(\text{pt})_2=1.4$, $\mu(\text{pt})_1=8, \sigma(\text{pt})_1=1.3$, $\mu(\text{pt})_1=7, \sigma(\text{pt})_1=1.2$, and $\mu(\text{pt})_1=6, \sigma(\text{pt})_1=1.1$. We further assume that search times for parking spaces/ delivery addresses are distributed according to $\mathcal{LN}\left(2.5, 1.5\right)$.

\color{safeblue}
\section{Method Analysis}\label{sec: analysis heuristic}

In this section, we take a closer look at the functionality of our method. First, we analyze the PDFT-algorithm. Then, we show the value of transfer learning and the impact of LNS-iterations.

To analyze how many iterations the PDFT-algorithm requires we measure for each considered partial decision the number of iterations before termination and note whether the partial decision was feasible or infeasible. 
We then plot for each number of iterations the relative percentage of partial decisions where the PDFT-algorithm already terminated. The results are shown in Figure~\ref{fig:PDFT}. The x-axis shows the number of iterations. The y-axis depicts the accumulated percentage of termination for eventually feasible (light grey) and infeasible (dark grey) partial decisions. We note that the algorithm terminates in $24.6\%$ of the cases before running any iterations because the analytical checks indicate that for the partial decision candidate, no feasible solution can be obtained. 
We further observe that for the majority of partial decisions ($75.8\%$), the algorithm terminates after at most one run. After five iterations, the process terminated for about $90.0\%$ of decisions. 
The number of partial decisions where no feasible decisions can be found after the 25 iterations and which are then declared infeasible is only $1.0\%$. In an additional experiment, we extended the number of iterations until also the remaining $1.0\%$ of decisions terminated. This revealed that only $0.29\%$ of all decisions were incorrectly declared infeasible when limiting the number of iterations to 25. In essence, while the PDFT-algorithm already solves within polynomial runtime, the analytical checks are quite powerful and reduce the runtime further while finding nearly all feasible decisions. Overall, the runtime of the PDFT-algorithm is rather negligible and decisions can be derived nearly instantaneously.

\begin{figure}[!t]
    \centering          \caption{Termination percentage of the PDFT-algorithm with respect to the number of PDFT-iterations.}
    \includegraphics[width = 15cm]{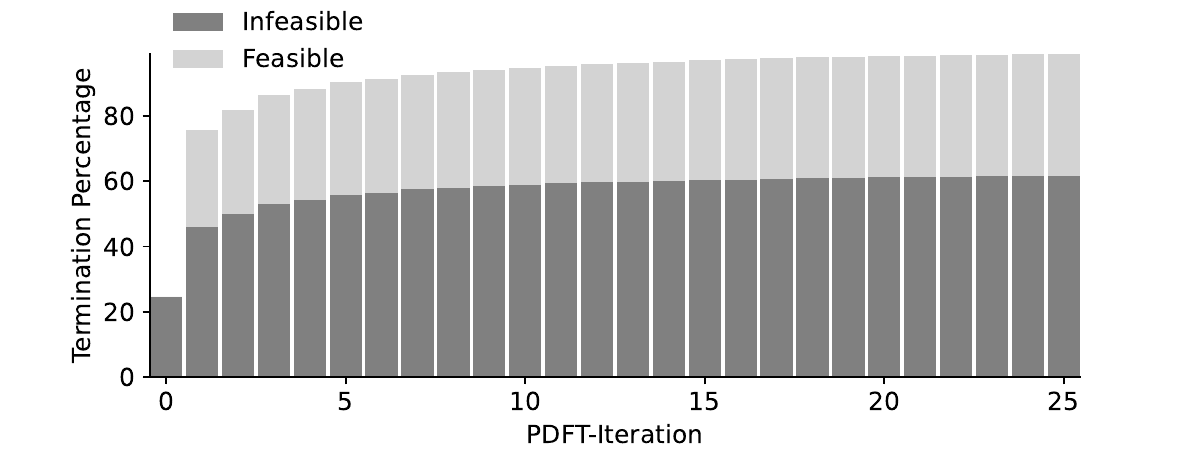}
    \label{fig:PDFT}
    \vspace{-0.5cm}
\end{figure}

We now analyze the importance of transfer learning and LNS-iterations during the learning. To this end, we introduce two benchmark policies:

\begin{itemize}
    \item $\textit{AI}_{\text{Small}}$: This policy is similar to $\textit{AI}$, but it is not fine-tuned for the Large instance.
    \item $\textit{AI}_{\text{30}}$: This policy is similar to $\textit{AI}$, but it is trained with 30 LNS-iterations only.
\end{itemize}

we compare our original policy $\textit{AI}$ (fine tuned using 70 LNS-iterations during offline training) to the benchmarks $\textit{AI}_{\text{Small}}$ (not fine tuned) and $\textit{AI}_{\text{30}}$ (fine tuned using 30 LNS-iterations during offline training) in case all policies perform 70 LNS-iterations during the online execution. The comparison is shown in the center of Figure~\ref{fig:iterations}. We observe that with $\textit{AI}_{\text{Small}}$ delay increases by $2.0\%$ and with $\textit{AI}_{\text{30}}$ by $4.1\%$. This indicates that both the increased number of LNS-iterations and especially the fine-tuning play important parts in the functionality of our methods.

Next, we analyze the performance of the three policies in case the online LNS-iterations are varied to 30 and 100. The delay increase compared to our \textit{AI} policy with 70 online LNS-iterations is shown on the left (for 30) and on the right (for 100 LNS-iterations) of Figure~\ref{fig:iterations}. 
Recalling that the improvement of our method to \textit{Integrated} was 12.8\%, we observe that with 30 LNS-iterations, only our \textit{AI} policy can compete. We also observe that \textit{AI} policy is superior regardless the number of online LNS-iterations. Policy $\textit{AI}_{\text{30}}$ achieves similar results to our policy $\textit{AI}$ only with 100 LNS-iterations, i.e., about 30\% more runtime. In contrast, even with 100 LNS-iterations, policy $\textit{AI}_{\text{Small}}$, i.e., without transfer learning, does not reach the high solution quality of $\textit{AI}$ with only 70 LNS-iterations. This comparison reveals two insights. First, while transfer learning already provides effective decisions in case of sufficient (online) LNS-iterations, fine tuning can further improve solution quality. Second, computational burden can be shifted between offline training and online implementation.

\begin{figure}[!t]
    \centering  \caption{Percentage increase in delay for different number of LNS-iterations and \text{AI} policies compared to \text{AI} with 70 LNS-iterations.}
    \includegraphics[width = 15cm]{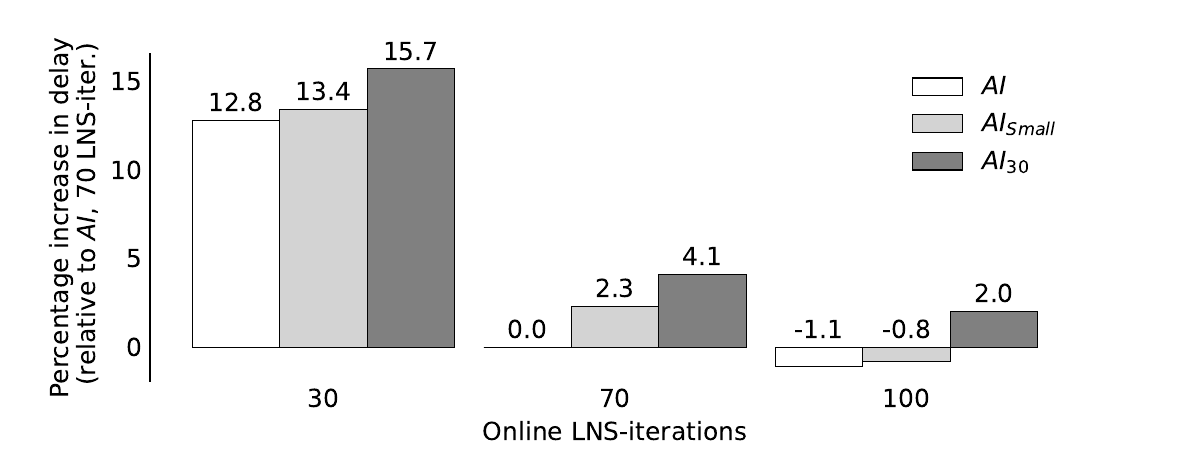}
    \label{fig:iterations}
\end{figure}

\color{safeblue}
\section{Workforce Utilization over Time}\label{app:A_util}

In the following, we analyze how the different policies utilize the workforce over time. To this end, we calculate the average workload utilization for each point of time, i.e., the percentage of the workforce that is occupied either by cooking or traveling. We do this for policies \textit{AI} and \textit{FIFO} and for the Large instances. The developments are shown in Figure~\ref{fig:L5 workload}. 
The x-axis shows the time, and the y-axis shows the percentage of utilized resources. The two demand peaks are clearly visible. We further observe that the vehicle resources are utilized more. The observed delay is the result of the complex dynamics of orders, cooks, and vehicles. Thus, it is not trivial to determine a clear bottleneck in the resources. 
Given the near 100\% utilization during the second peak, for this instance, the vehicles are more likely the bottleneck causing delay. 
However, we point out that restaurants have different expected preparation times and we note that the cooks of the \enquote{slowest} restaurant achieve 
an average utilization of approximately 85\%, thus possibly contributing to the delay as well.

We observe that the utilization during the first, smaller peak differs only slightly between the two policies. The vehicle utilization is slightly smaller for \textit{AI}, likely due to increased consolidation. At the beginning of the second peak, we start observing significant differences between the two policies. For \textit{AI}, the utilization of both resources is higher compared to FIFO during the peak, but drops significantly faster at its end. This indicates that with anticipatory and integrated decision making, order throughput can be be kept high resulting in a comparably small backlog. With FIFO, especially the cooks are underutilized leading to longer working times and more delay.

\begin{figure}[!t]
    \centering
    \caption{\color{safeblue} Average workload utilization of vehicles and cooks by AI and FIFO over time, Large instances}
    \includegraphics[width=0.9\textwidth]{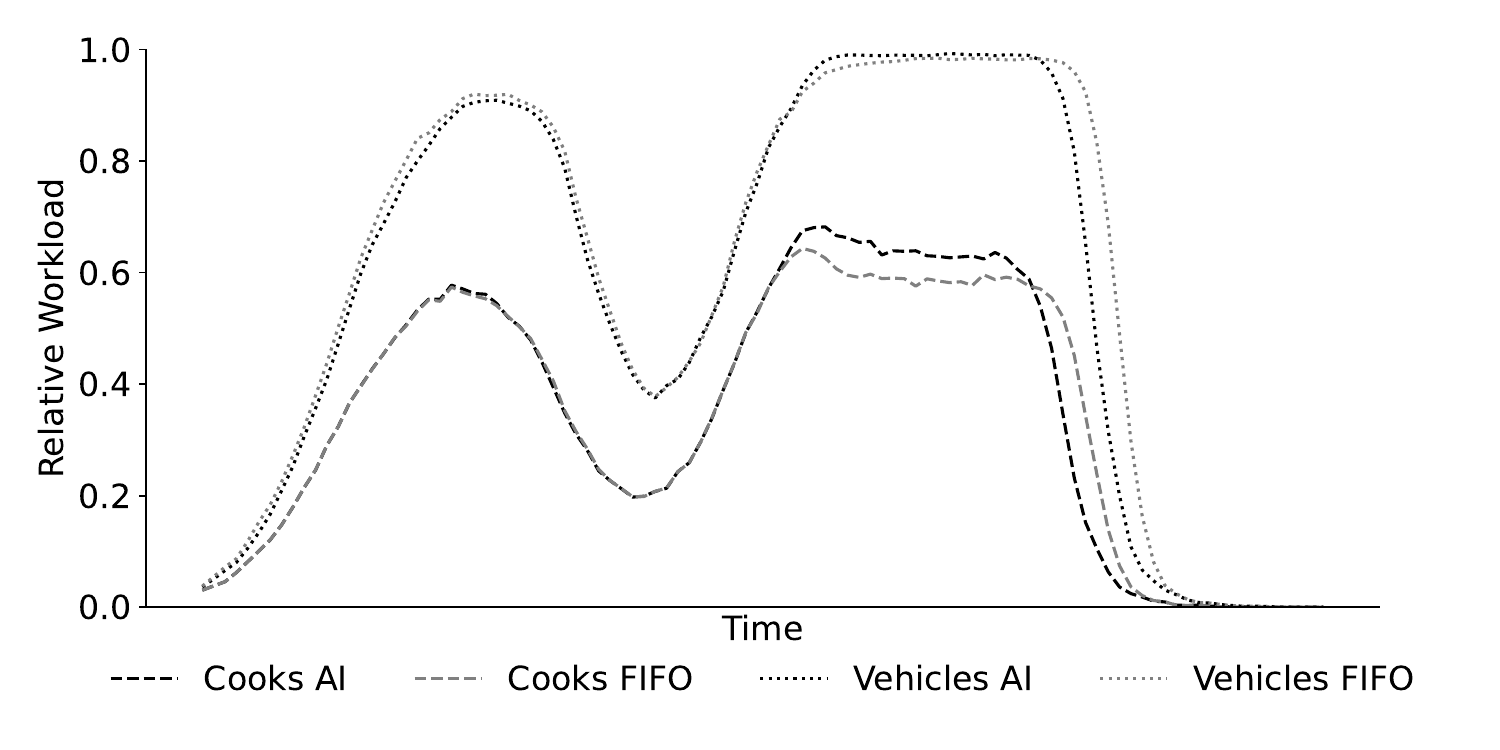}
    \label{fig:L5 workload}
\end{figure}


\color{black}

\section{Additional Results}\label{app:A_results}
In this section, we present additional results. First, we examine the delay of close and far orders. Table~\ref{tab: distance orders res} summarizes the results. \textit{AI} improves the average delay over both far and close orders, compared to \textit{Integrated} and \textit{FIFO} methods. However, both the average delay and the average delay of the late orders is not fair between close and far orders. 
This can be expected as the net travel time has a significant impact on the overall delivery time. A future study might consider incorporating fairness in the model.

\begin{table}[!t]
\caption{Delay (minutes) by distance from the depot: "Close orders" and "far orders" defined as less and more than 10 minutes drive from the depot, respectively.}\label{tab: distance orders res}
\centering
\small
\begin{tabular}{lrrrr}
\hline
\multirow{2}{*}{Solution method} & \multicolumn{4}{c}{Average delay (in min.)} \\
 & \multicolumn{1}{l}{Close orders} & \multicolumn{1}{l}{Far orders} & \multicolumn{1}{l}{Late close orders} & \multicolumn{1}{l}{Late far orders} \\ \hline
\textit{FIFO}                             & 14.5     & 16.9     & 32.4    & 32.6    \\ 
\textit{Integrated}                           & {11.9}    & {14.5}     & {28.0}    & {29.0}    \\
\textit{AI}  \                  &   9.8    &   12.6   &   23.4   &    25.0  \\\hline
\end{tabular}
 \vspace{-3mm}
\end{table}

Another direction we have studied is the average delay of orders of different food types. We note that the average number of orders from each food type is similar and thus the only difference between them is the preparation distribution, where food types 1 to 5 are ordered from the longest average preparation times to the shortest. Table \ref{tab:results foodtypes} summarizes the average delays of the different food types by the different methods and the relative improvement of \text{AI} over the benchmark solution methods. 
Here too, we observe that more time-consuming orders obtain larger average delay, i.e., the delay grows with the average preparation times.
We also observe that our policy improves the average delay over all food types.

\vspace{-\abovedisplayskip}
\begin{table}[!t]
\caption{Delay by food type (minutes)}
\label{tab:results foodtypes}
\centering
\small
\begin{tabular}{lrrrrr}
\hline
Food Type &
  \multicolumn{1}{l}{\textit{FIFO}} &
  \multicolumn{1}{l}{\textit{Integrated}} &
  \multicolumn{1}{l}{\textit{AI}} &
  \multicolumn{1}{l}{\begin{tabular}[l]{@{}l@{}}Imp. over \\ \textit{FIFO} \%\end{tabular}} &
  \multicolumn{1}{l}{\begin{tabular}[l]{@{}l@{}}Imp. over \\ \textit{Integrated} \%\end{tabular}} \\ \hline
  1 & 16.8 & 14.7 & 13.0 & 29.2 & 13.1\\ 
  2 & 15.3 & 12.8 & 10.9 & 40.4 & 17.4\\ 
  3 & 14.9 & 12.2 & 10.0 & 49.0 & 22.0\\ 
  4 & 14.8 & 12.1 & 9.9  & 49.5 & 22.2\\ 
  5 & 14.6 & 11.9 & 9.8  & 49.0 & 21.4\\ 
 \hline 
\end{tabular}
\end{table}
\FloatBarrier

\end{appendices}

\end{document}